\documentclass{article} 
\usepackage{iclr2024_conference, times}
\usepackage{hyperref}

\usepackage{amsmath,amsfonts,bm}

\def\eqref#1{equation~\ref{#1}}

\def\1{\bm{1}}

\DeclareMathAlphabet{\mathsfit}{\encodingdefault}{\sfdefault}{m}{sl}
\SetMathAlphabet{\mathsfit}{bold}{\encodingdefault}{\sfdefault}{bx}{n}

\usepackage{microtype}
\usepackage{graphicx}
\usepackage{booktabs} 

\usepackage{amsmath}
\usepackage{amssymb}
\usepackage{mathtools}
\usepackage{amsthm}
\usepackage{bm}
\usepackage{graphicx}
\usepackage{multirow}
\usepackage[normalem]{ulem}
\useunder{\uline}{\ul}{} 
\usepackage{amsthm}
\usepackage{url}
\usepackage[capitalize,noabbrev]{cleveref}

\theoremstyle{plain}
\newtheorem{theorem}{Theorem}[section]
\newtheorem{proposition}[theorem]{Proposition}

\theoremstyle{definition}
\newtheorem{definition}[theorem]{Definition}

\theoremstyle{remark}

\usepackage{multirow}
\usepackage{wrapfig}

\usepackage{caption,subcaption}
\usepackage[textsize=tiny]{todonotes}

\usepackage{xcolor}

\title{Quantifying and Enhancing Multi-modal Robustness with Modality Preference}

\author{Zequn Yang$^1$, Yake Wei$^1$, Ce Liang$^1$, Di Hu$^1$\thanks{Corresponding author} \\
 Gaoling School of Artificial Intelligence, Renmin University of China$^1$ \\
\texttt{\{zqyang,yakewei,liangce158,dihu\}@ruc.edu.cn}
} 

\iclrfinalcopy 
\begin{document}

\maketitle

\vspace{-3mm}
\begin{abstract}
\vspace{-3mm}
Multi-modal models have shown a promising capability to effectively integrate information from various sources, yet meanwhile, they are found vulnerable to pervasive perturbations, such as uni-modal attacks and missing conditions. To counter these perturbations, robust multi-modal representations are highly expected, which are positioned well away from the discriminative multi-modal decision boundary. In this paper, different from conventional empirical studies, we focus on a commonly used joint multi-modal framework and theoretically discover that larger uni-modal representation margins and more reliable integration for modalities are essential components for achieving higher robustness. This discovery can further explain the limitation of multi-modal robustness and the phenomenon that multi-modal models are often vulnerable to attacks on the specific modality. Moreover, our analysis reveals how the widespread issue, that the model has different preferences for modalities, limits the multi-modal robustness by influencing the essential components and could lead to attacks on the specific modality highly effective. Inspired by our theoretical finding, we introduce a training procedure called \textit{Certifiable Robust Multi-modal Training} (CRMT), which can alleviate this influence from modality preference and explicitly regulate essential components to significantly improve robustness in a certifiable manner. Our method demonstrates substantial improvements in performance and robustness compared with existing methods. Furthermore, our training procedure can be easily extended to enhance other robust training strategies, highlighting its credibility and flexibility. The code is available at https://github.com/GeWu-Lab/Certifiable-Robust-Multi-modal-Training. 

\end{abstract}

\vspace{-4mm}
\section{Introduction}
\vspace{-1mm}
As data are often presented from different perspectives, like text, images, and audio, how to effectively exploit and integrate information from multiple sources becomes important. This has given rise to the concept of multi-modal learning, which serves as a potent approach that enables a more comprehensive understanding of complex concepts and facilitates more effective knowledge acquisition for different sources of information \citep{wei2022learning}.
Nowadays, multi-modal learning has demonstrated its remarkable ability in various tasks, including scene understanding~\citep{antol2015vqa, yang2022avqa, li2022learning}, and emotion recognition~\citep{tripathi2018multi,chudasama2022m2fnet}.

\vspace{-1mm}
However, real-world data are often perturbed, such as attacks, and missing modality, which impacts the performance of multi-modal model~\citep{kumar2020finding}. 
As a result, multi-modal robustness, which refers to the model's ability to defend against such perturbations, has received increasing attention in recent studies~\citep{bednarek2020robustness, vishwamitra2021understanding}. 
Unlike data from a single modality, multi-modal data can be perturbed across all modalities. Therefore, for robustness, multi-modal models should have the ability to resist attacks on both individual and multiple modalities.
Nevertheless, experiments have suggested that multi-modal models could perform badly when encountering perturbation~\citep{noever2021reading} or missing modality~\citep{yu2020investigating, ma2022multimodal}. 
Based on these observations about the vulnerability of the multi-modal model, how to achieve a robust multi-modal model arises and attaches more focus. 

\vspace{-1mm}
To address this, previous methods improve the training strategies to obtain a robust multi-modal model~\citep{liang2021multibench, ding2021multimodal}. Specifically, certain studies extend uni-modal robust training strategies, like adversarial training and mixup, to learn a discriminative multi-modal decision boundary~\citep{li2022adversarial, maheshwari2023missing}. Others take steps like cross-modal alignment~\citep{tian2021can} to enhance the connection among modalities and obtain compact and robust multi-modal representation. 
However, even if they empirically improve robustness to some degree, they still lack in-depth theoretical analysis to understand the resilience of multi-modal models to perturbations, which is vital for safety-critical applications. More importantly, a universal phenomenon can be observed for these robust training methods that the adversarial attack on specific modalities could be more effective than others. As shown in \autoref{teasor}, the $\ell_2$-PGD attack is more effective on modality \#$a$ than modality \#$v$ for both Joint Training and three widely used robust training methods on the Kinetics Sounds dataset. 

\vspace{-1mm}
\begin{wrapfigure}{rhtp!}{0.45\linewidth}
\vspace{-7.2mm}
\includegraphics[width=\linewidth]{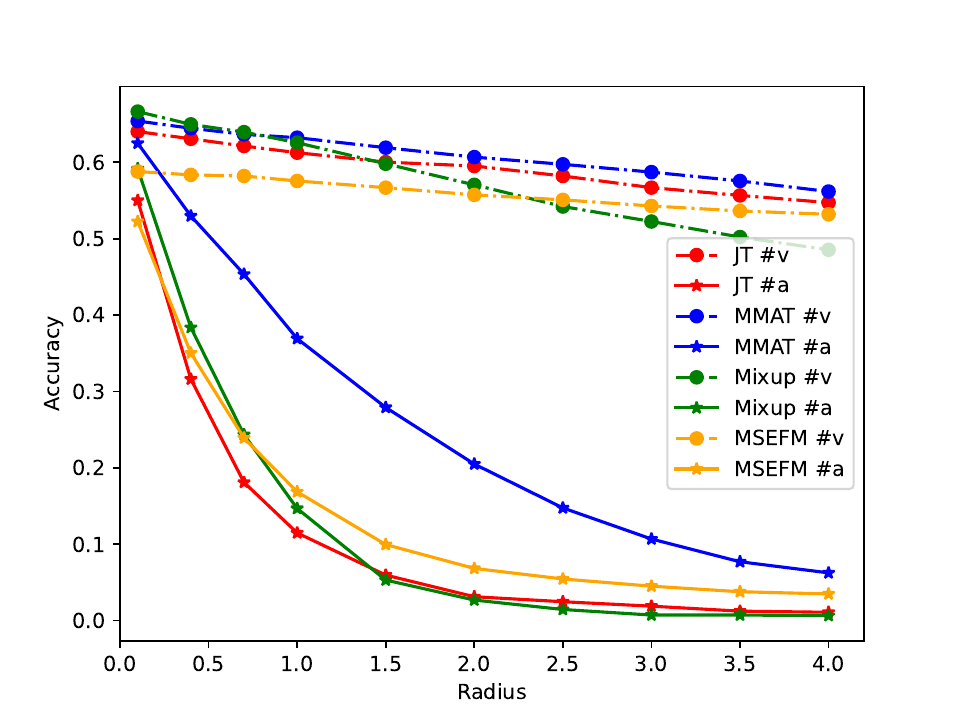}
\vspace{-7mm}
\caption{Accuracy of different multi-modal robust training methods compared with Joint Training (JT) baseline under $\ell_2$-PGD attack with a range of radius for modality \#$v$ (vision) and \#$a$ (audio) respectively on Kinetics Sounds dataset. Results show that all these methods are more vulnerable to attacks on the specific modality \#$a$.}
\vspace{-5mm}
  \label{teasor}
\end{wrapfigure}

To deeply understand multi-modal robustness and elucidate this phenomenon, we first depict the multi-modal decision boundary by integrating uni-modal representation margins. Then, we derive a lower bound for the perturbation radius that the multi-modal model can consistently defend against. We discover that larger uni-modal margins coupled with reasonable integration are crucial for enhanced multi-modal robustness.
With the above theoretical results, we investigate the pervasive issue~\citep{wang2020makes, peng2022balanced} that multi-modal models exhibit a pronounced preference for a particular modality. This preferred modality profoundly influences the model's decision, then resulting in the above robustness phenomenon.
On one hand, when a specific preferred modality is substantial enough to rely upon, the model will be reluctant to learn from other modalities, leading to the imbalance problem~\citep{huang2022modality, wu2022characterizing}. This imbalance problem hinders the enhancement of the uni-modal margin on the other modality, thus limiting the robustness. On the other hand, since the multi-modal model heavily relies on the preferred modality, the corresponding factor used in modality integration becomes larger, which will amplify the variation of the uni-modal margin in decision-making. Hence, in case the preferred modality is vulnerable, the multi-modal model becomes more vulnerable to multi-modal attack. Further, this preference for vulnerable modality makes attacks on the preferred modality significantly more effective than other ones, explaining the observation in~\autoref{teasor}. 

\vspace{-1mm}
Since the essential components of robustness have an interrelation with each other, directly applying regulation~is hard to guarantee higher robustness. 
To address this, we employ an orthogonal-based framework that formulates an alternative bound, which eliminates the interrelation and explicitly presents the integration. 
Building upon our theoretical analysis, we introduce a two-step \textit{Certifiable Robust Multi-modal Training} (CRMT) procedure to ensure progressively superior robustness.
Initially, we redefine uni-modal margins related to the reliability of the modality in this framework. Then we propose to regulate the unreliable modality by enlarging its margin, which can alleviate the imbalanced problem brought by modality preference. 
Further, our approach adjusts the integration of modalities considering the improvement of certified bound. These steps not only mitigate the large gap between robustness against attack on each modality but also credibly guarantee higher multi-modal robustness. 
To validate our method, we conduct extensive experiments to present the advanced robustness against both uni-modal and multi-modal attacks. Our contributions are as~follows:
\vspace{-2mm}
\begin{enumerate}
    \item We focus on a commonly used multi-modal model, offering invaluable insights into the essential components influencing multi-modal robustness.
\vspace{-1mm}
    \item We present analyses highlighting how multi-modal preference limits the multi-modal robustness and contributes to the vulnerability of multi-modal models towards specific modalities.
\vspace{-5mm}
    \item Drawing from our theoretical findings, we introduce a two-step training procedure, alleviating the limitation brought by modality preference. Our method can effectively enhance both performance and robustness over three real-world multi-modal datasets.
\end{enumerate}

\vspace{-2mm}

\vspace{-2mm}
\section{Related work}

\vspace{-1mm}

\paragraph{Multi-modal Robustness Analysis.}
Recent studies have highlighted the vulnerability of deep neural networks (DNNs) to attacks and perturbations~\citep{goodfellow2014explaining, madry2017towards}, raising significant concerns regarding their deployment. With the presence of multiple modalities, the forms of perturbations include uni-modal attacks,  multi-modal attacks~\citep{schlarmann2023adversarial}, and modality missing~\citep{lee2023multimodal}.
One possible way to enable the models to resist these perturbations is to design robust training strategies to obtain a reliable decision boundary that is distanced from samples~\citep{li2021audio}. Certain methods, such as multi-modal adversarial training~\citep{li2022adversarial} and incorporating uni-modal tasks~\citep{ma2022multimodal}, have proven pertinent in this context.
On the other hand, studies focus on the latent consistency among modalities and propose to enhance modality interaction like alignment~\citep{tsai2018learning} to obtain compact and robust representation~\citep{bednarek2020robustness}. Among these empirical works, we identify a universal phenomenon that multi-modal models are commonly vulnerable to a certain modality~\citep{liang2021multibench}, which currently lacks a comprehensive theoretical explanation. This motivates us to theoretically figure out the essential components determining the multi-modal robustness and explain the vulnerable modality phenomenon.

\vspace{-2mm}
\paragraph{Multi-modal imbalance problem.}
Multi-modal learning is a significant approach for achieving a richer understanding across various tasks~\citep{baltruvsaitis2018multimodal, jiang2021review}. 
However, even with the potential for a more comprehensive latent representation from multi-modal models~\citep{huang2021makes}, their performance might not always surpass that of the best uni-modal counterparts~\citep{wang2020makes, huang2022modality}. Due to the inherent differences between modalities, a multi-modal model might prefer or lean towards a modality that is easier to learn, potentially overlooking others~\citep{wu2022characterizing}. In response, various strategies are proposed to strengthen the learning of uni-modalities, thereby enhancing the generalization of multi-modal model~\citep{wang2020makes, peng2022balanced, fan2023pmr, xu2023mmcosine}. 
However, these researches do not elucidate how this modality preference problem impacts robustness against perturbations, which is our main focus.

\vspace{-2mm}

\paragraph{Certified robustness.}

To describe the robustness of the model, certified robustness is introduced to describe the size of permissible perturbations that a model can consistently defend against. 
Techniques such as randomized smoothing~\citep{cohen2019certified, rosenfeld2020certified, yang2021certified, salman2019provably} and interval bound propagation~\citep{zhang2018efficient, lyu2021towards, xu2020automatic} are widely utilized to relax the model, which assists in determining the certificate bounds. Additionally, the Lipschitz constant can serve as an intuitive metric to illustrate how perturbations influence the decision of models~\citep{weng2018towards, leino2021globally}.
However, these methods are tailored for inputs with only a single modality but fail to focus on the primary concerns with multiple modalities.
In our research, we establish certified robustness for multi-modal models and determine the essential components influenced by modality preference that limit the robustness.


\section{Method}
\subsection{Preliminaries}

\paragraph{Multi-modal Framework. }

We consider a general $K$-way classification problem with a vectorized input sample $\bm{x} = (\bm{x}^{(1)},\bm{x}^{(2)}) \in \mathbb{R}^{d_1 + d_2}$ consisting of two modalities, and a ground truth label $y \in [K]$. We consider the commonly used joint learning framework in multi-modal learning, where all of the modalities are projected into the shared space for downstream tasks~\citep{huang2022modality}. Concretely, the uni-modal representations are extracted by encoders $\phi^{(m)}$ and concatenated to form the joint representation, which is then mapped to the output space using a linear classifier, where $W \in \mathbb{R}^{K \times (\mathrm{dim}(\phi^{(1)}) +  \mathrm{dim}(\phi^{(2)})) }$ and $\bm{b} \in \mathbb{R}^K $ are the weight matrix and bias respectively, and $\mathrm{dim}(\cdot)$ represents the dimension. 
The logits output of the multi-modal model can be denoted as $h(\bm{x}) = W [\phi^{(1)} (\bm{x}^{(1)}); \phi^{(2)} (\bm{x}^{(2)})] + \bm{b}= W^{(1)} \phi^{(1)} (\bm{x}^{(1)}) + W^{(2)} \phi^{(2)} (\bm{x}^{(2)}) + \bm{b}$, where $W^{(m)} \in \mathbb{R}^{K \times \mathrm{dim}(\phi^{(m)})} $ is the part of classifier $W$ related to the $m$-th modality.

\vspace{-1mm}

\begin{figure*}[t]
\vspace{-2mm}
  \centering
  \setlength{\abovecaptionskip}{3mm}
  \setlength{\belowcaptionskip}{-4mm} 
  
  \includegraphics[width=\textwidth]{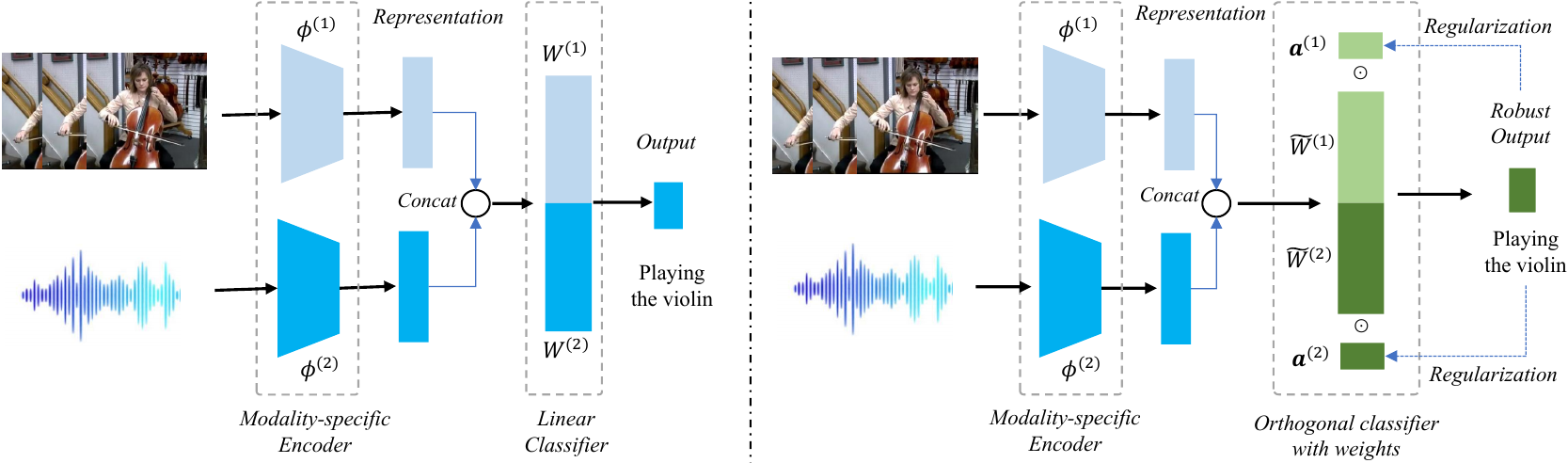}
  \caption{Illustration of traditional multi-modal joint learning framework ~\citep{baltruvsaitis2018multimodal} (left) and our framework introducing orthogonality into each uni-modal classifier (right). Our framework can be easily applied to explicit regularization to achieve larger certified robustness. 
    }
\vspace{-2mm}
  \label{method}
\end{figure*}

\vspace{-1mm}

\paragraph{Multi-modal robustness.}

To assess the certified robustness of a multi-modal model, we can measure the radius of the smallest perturbation that shifts a sample to the decision boundary. In other words, any perturbation that falls within this radius can be defended against.
\begin{definition}
Suppose the multi-modal model $h$ can correctly classify the sample $\bm{x}$, \textit{i.e.} $\forall k\neq y, h_y(\bm{x}) > h_k(\bm{x})$, where $h_k$ denotes the logit score of the $k$-th class. The robustness radius of the multi-modal model towards the sample $\bm{x}$ can be defined as:
\begin{equation}
\begin{aligned}
    P(\bm{x}) = \min_{\bm{x}'} \left\|\bm{x} - \bm{x}' \right\|_2 \quad \quad
    s.t. ~\exists j \neq y,   h_y (\bm{x}') = h_j (\bm{x}').
\end{aligned}
\label{mm_robustness}
\end{equation}
\end{definition}
\vspace{-3mm} 
\autoref{mm_robustness} seeks to identify the smallest perturbation, denoted as $\bm{x} - \bm{x}'$, that results in reaching the decision boundary between the ground truth $y$ and its nearest class $j$. 
Hence, any smaller perturbation can always be defended against. 
Here, we use $\ell_2$-norm to measure the radius, reflecting the overall size of the perturbation~\citep{tsuzuku2018lipschitz, carlini2017towards}. 

\vspace{-1mm}
\subsection{Certified robustness for Multi-modal Model}

\vspace{-1mm}
In this section, we endeavor to uncover the essential components that impact multi-modal robustness.
Referring to \autoref{teasor}, there is a notable variation in the effectiveness of perturbations on different modalities. This observation prompts us to distinguish the differences within each uni-modality. Therefore, we introduce the concept of \textit{uni-modal margin}, which quantifies the distance between a uni-modal representation and the uni-modal decision boundary. To elaborate further, $(W^{(m)}_{y\cdot} - W^{(m)}_{k\cdot})\phi(\bm{x}^{(m)})=0$ signifies that sample $\bm{x}^{(m)}$ is positioned on the decision boundary between classes $y$ and $k$ for $m$-th modality. The uni-modal representation margin can be defined as:
\vspace{-1mm}
\begin{definition}
    Given the uni-modal encoder $\phi^{(m)}$ and the classifier $W^{(m)}$, the margin on representation space between ground truth $y$ and other label $k$ is defined as:
    \begin{equation}
    \zeta_{k}^{(m)}(\bm{x}^{(m)}) = \frac{(W^{(m)}_{y\cdot} - W^{(m)}_{k\cdot})  \phi^{(m)} (\bm{x}^{(m)})}{\left\|W^{(m)}_{y\cdot} - W^{(m)}_{k\cdot}\right\|_2}.  
\end{equation}
\end{definition}
\vspace{-4mm}
In this context, a larger margin indicates the uni-modality is more reliable in distinguishing these two classes. Using this margin definition, we can re-examine the constraint conditions in \autoref{mm_robustness}, which outline the perturbed sample located on the multi-modal decision boundary where $y$ and the closest class $j$ are tied in the output space:

\vspace{-4mm}
\begin{equation}
\begin{aligned}
   h_y(\bm{x}') - h_j(\bm{x}’) = c_j^{(1)} \zeta_{j}^{(1)}(\bm{x}'^{(1)}) + c_j^{(2)} \zeta_{j}^{(2)}(\bm{x}'^{(2)}) + \beta_j = 0.
\end{aligned}
\label{db_eq5}
\end{equation}
\vspace{-4mm}

Inspired by \autoref{db_eq5}, we observe that the multi-modal decision boundary can be described as integration of different uni-modal margins of the perturbed sample with factors $c_j^{(m)} = \|W^{(m)}_{y\cdot} - W^{(m)}_{j\cdot}\|_2$ and constant term $\beta_j = b_y - b_j$.
The integration factors $c_j^{(m)}$ quantify how variations in the uni-modal margin influence multi-modal decisions. When perturbing a sample $\bm{x}$ towards the decision boundary, a larger factor $c_j^{(m)}$ indicates altering the margin of the corresponding modality is more effective.
Meanwhile, when the sample is perturbed, how the uni-modal margin varies also depends on the size of the perturbation.
Thus, we propose to introduce the Lipschitz constant $\tau_{j}^{(m)}$ for the uni-modal margin. This Lipschitz constant is the smallest constant to limit the local variation range of a certain function~\citep{finlay2018improved}, which is given by:

\vspace{-2mm}
\begin{equation}
\begin{aligned}
    |\zeta_{j}^{(m)}(\bm{x}^{(m)}) -\zeta_{j}^{(m)}(\bm{x}'^{(m)})|  \leq  \tau_{j}^{(m)}\left\|\bm{x}^{(m)} - \bm{x}'^{(m)}\right\|_2.
\end{aligned}
\label{Lip_0}
\end{equation}

Then, we can provide the lower bound for the perturbation radius for multi-modal robustness.
\begin{theorem} 
Given an input $\bm{x}$ with ground-truth label $y \in [K]$ and the closest label $j \neq y$, $\zeta_{j}^{(m)}(\bm{x}^{(m)})$ as the representation margin for $m$-th modality with Lipschitz constraint $\tau_j^{(m)}$, and the integration factor $c^{(m)}_j$. Define $\bm{x}'$ as the perturbed sample, and $\bm{x} - \bm{x}'$ as the perturbation. The lower bound for the perturbation radius can be described as:
\begin{equation}
\begin{aligned}
 P(\bm{x})& = \min_{\bm{x}'} \left\|\bm{x} - \bm{x}'\right\|_2 
 \geq \frac{c_j^{(1)} \zeta_{j}^{(1)}(\bm{x}^{(1)}) + c_j^{(2)} \zeta_{j}^{(2)}(\bm{x}^{(2)})+ \beta_j}{\sqrt{(c_j^{(1)} \tau_j^{(1)})^2 +(c_j^{(2)} \tau_j^{(2)})^2 }} \\
  where &  \quad j \neq y \quad s.t. \quad  c_j^{(1)} \zeta_{j}^{(1)}(\bm{x}'^{(1)}) + c_j^{(2)} \zeta_{j}^{(2)}(\bm{x}'^{(2)}) + \beta_j = 0.
\end{aligned}
\label{ori_robustness}
\end{equation}
\end{theorem}


The detailed proof can be found in the Appendix 7.1. Based on the certified bound above, we deduce that the multi-modal robustness depends on three primary factors: uni-modal representation margins $\zeta_j^{(m)}$, integration $c_j^{(m)}$, and bias difference $\beta_j$.
Firstly, robustness increases proportionally with the enhancement of the uni-modal representation margin $\zeta_j^{(m)}$. Secondly, the integration $c_j^{(m)}$ is related to both uni-modal margins and the uni-modal Lipschitz constant, the expected integration for robustness requires considering both. Thus, a reasonable choice of integration factor can benefit higher robustness. 
Thirdly, when the sample is perturbed, the bias difference term $\beta_j$ stays invariant, since it depends on class $y$ and $j$ rather than the specific sample. Based on this fact, the bias difference is not considered in our analysis of model robustness. 
In a nutshell, we recognize that uni-modal margins and the integration of modalities are two essential components.
In the following section, we will analyze how these essential components vary and influence the robustness of multi-modal models, especially under modality preference. 

\subsection{Analysis about modality preference for robustness}

Since modalities have different amounts of information, some specific modalities might hold more significance than others in decision-making processes~\citep{gat2021perceptual}. As a result, it is widely recognized that multi-modal models tend to show a preference for and predominantly rely on specific modality~\citep{huang2022modality, wu2022characterizing}. However, this preference or over-reliance on the specific modality poses challenges for achieving a robust multi-modal model. We delve deeper into how such preferences impact the two essential components of the certified bound in \autoref{ori_robustness}.

\vspace{-3mm}
\paragraph{Uni-modal representation margin.}  As shown in \autoref{ori_robustness}, a larger uni-modal margin $\zeta_{j}^{(m)}(\bm{x}^{(m)})$ determines higher certified robustness. However, when the learned information in the preferred modality is sufficiently reliable, the multi-modal model is reluctant to learn more information from other modalities~\citep{wu2022characterizing, huang2022modality}. Thus, the modality preference leads to an imbalance problem hindering the development of uni-modal representations, resulting in a narrower representation margin and ultimately constraining the certified robustness of the multi-modal model.

\vspace{-3mm}

\paragraph{Integration of modalities.} 
According to modality preference, the decision of the multi-modal model highly depends on a specific modality. Thus, considering the integration of modalities, the preferred modality contributes more and is allocated a larger integration factor $c_j^{(m)}$, which could amplify the variation of the uni-modal margin in multi-modal decision-making. Since the preference is only determined by whether the modality with ideal discriminative ability, the multi-modal model could prefer a vulnerable modality, which has a larger $\tau_j^{(m)}$. Thus, the perturbation for this preferred but vulnerable modality leads to larger variations in multi-modal margins, which is further amplified in decision-making. Motivated by this phenomenon, we define $\eta^{(m)} = c_j^{(m)} \tau_j^{(m)}$ as the vulnerability indicator of modality $m$. When the multi-modal model exhibits a preference for a vulnerable modality $1$, there is a significant imbalance in this indicator, with $\eta^{(1)} \gg \eta^{(2)}$. Consequently, we observe:

\begin{equation}
P(\bm{x}) \geq \frac{c_j^{(1)} \zeta_{j}^{(1)}(\bm{x}^{(1)}) + c_j^{(2)} \zeta_{j}^{(2)}(\bm{x}^{(2)})+ \beta_j}{\sqrt{( \eta^{(1)})^2 +(\eta^{(2)})^2 }} \approx \frac{c_j^{(1)} \zeta_{j}^{(1)}(\bm{x}^{(1)}) + c_j^{(2)} \zeta_{j}^{(2)}(\bm{x}^{(2)})+ \beta_j}{\eta^{(1)} }. 
\label{approx_uni-modal}
\end{equation}

That is to say, the modality preference on vulnerable modality leaves an unreasonable imbalance on $\eta^{(m)}$, thus the multi-modal robustness is highly dependent on the modality with a larger vulnerability indicator. In this way, distinctly attacking the vulnerable modalities alone is enough to obscure the model. Here we further provide the multi-modal robustness under uni-modal attack case:
\begin{proposition}
Following the setting in Theorem 3.3, \emph{w.l.o.g.} considering the perturbation on $1$-th modality $\bm{x}'^{(1)}$, the lower bound for the uni-modal perturbation radius can be described as: 
\begin{equation}
\begin{aligned}
 \min_{\bm{x}'^{(1)}}& \left\|\bm{x}^{(1)} - \bm{x}'^{(1)} \right\|_2 
 \geq \frac{c_j^{(1)} \zeta_{j}^{(1)}(\bm{x}^{(1)}) + c_j^{(2)} \zeta_{j}^{(2)}(\bm{x}^{(2)})+ \beta_j}{c_j^{(1)} \tau_j^{(1)}} \\
  where &  \quad j \neq y \quad s.t. \quad  c_j^{(1)} \zeta_{j}^{(1)}(\bm{x}'^{(1)}) + c_j^{(2)} \zeta_{j}^{(2)}(\bm{x}^{(2)}) + \beta_j = 0.
\end{aligned}
\label{uni-modal-perturbation}
\end{equation}
\end{proposition}
\vspace{-2mm}

The lower bounds of different uni-modal perturbations share identical numerators but differ in denominators, the vulnerability indicator $\eta^{(m)}$. The larger indicator on preferred modality reduces the lower bound for perturbations on this modality, thus making attacks on this preferred modality more effective, explaining the observation in \autoref{teasor}.

\vspace{-1mm}

\subsection{Certifiable Robust Multi-modal Training}

\vspace{-1mm}

Based on the above analysis, we target to improve the uni-modal representation margin $\zeta_j^{(m)}$ and adjust the integration of modalities $c_j^{(m)}$ as regulation for higher certified robustness. 
However, these regulations are intricately linked with the linear classifier $W$, potentially leading to conflicts in optimization objectives. As a result, stably enhancing the certified robustness presents challenges. 
Additionally, determining the Lipschitz constant for the margin is computationally demanding. 
To address these challenges, we propose to adopt orthogonality ~\citep{huang2018orthogonal} within uni-modal linear classifiers as our framework.
Detailedly, we ensure that in each modality, the class-specific vectors $\tilde{W}^{(m)}_{k\cdot}, k\in [K]$ are unit and orthogonal. 
Since the valuable information between modalities is different, we apply the weight $\bm{a}^{(m)}\in \mathbb{R}^{K}$ to lead the model focusing on more reliable modalities. Integrating these weights can enhance the model's ability to effectively utilize the valuable information from each modality.
Therefore, learning of uni-modal representations and integration of modalities can be decoupled. The score corresponding to the $k$-th class can be expressed as:
\begin{equation}
    \begin{aligned}
         \tilde{h}_k(\bm{x}) =  a_k^{(1)} \tilde{W}_{k\cdot}^{(1)} \phi^{(1)} (\bm{x}^{(1)}) +  a_k^{(2)}\tilde{W}_{k\cdot}^{(2)} \phi^{(2)} (\bm{x}^{(2)}) + \tilde{b}_k, 
    \end{aligned}
\end{equation}
where $\tilde{W}^{(m)} \in \mathbb{R}^{K \times \mathrm{dim}(\phi^{(m)})}$ is the matrix with orthogonal rows, satisfying $\tilde{W}^{(m)} (\tilde{W}^{(m)})^T = I_K$. We can use $\tilde{W}_{k\cdot}^{(m)}\Phi(\bm{x}^{(m)})$ as the uni-modal score, which represents whether the uni-modal representation is well learned toward the $k$-th class. 
With this framework in place, we are able to express the new Lipschitz constant $\tilde{\tau}_k^{(m)}$ for the uni-modal score on class $k$ in the $m$-th modality as:
\vspace{-1mm}
\begin{equation}
    |\tilde{W}^{(m)}_{k\cdot} \phi^{(m)} (\bm{x}^{(m)})- \tilde{W}^{(m)}_{k\cdot} \phi^{(m)} (\bm{x}'^{(m)}) | \leq \tilde{\tau}_k^{(m)} \left\|\bm{x}^{(m)} - \bm{x}'^{(m)}\right\|_2.
\vspace{-1mm}
\end{equation}
With these definitions, we can derive the certified bound for the multi-modal perturbation radius, which is distinctly tailored to the framework employing orthogonal classifiers:
\begin{theorem} 
Given an input $\bm{x}$ with ground-truth label $y \in [K]$ and the closest label $j \neq y$, the orthogonal classifier $\tilde{W}^{(m)}$, the modality-specific weight $\bm{a}^{(m)}$, the Lipschitz constant $\tilde{\tau}_j^{(m)}$, and the difference of the bias $\tilde{\beta}_j =  \tilde{b}_y -  \tilde{b}_j$. The lower bound for the perturbation radius with the orthogonal-based framework can be described as:
\begin{equation}
\begin{aligned}
    P(\bm{x}) \geq &\frac{\sum_{m=1}^2 \left( a_y^{(m)}\tilde{W}^{(m)}_{y\cdot} \phi^{(m)} (\bm{x}^{(m)})- a_j^{(m)}\tilde{W}^{(m)}_{j\cdot} \phi^{(m)} (\bm{x}^{(m)})\right) + \tilde{\beta}_j  }{\sqrt{\sum_{m=1}^2 \left( a^{(m)}_y \tilde{\tau}^{(m)}_y + a^{(m)}_j \tilde{\tau}^{(m)}_j\right)^2  }}  \\
  &\quad \quad  where  \quad  j \neq y, \quad s.t. \quad \tilde{h}_y(\bm{x}') = \tilde{h}_j(\bm{x}').
\end{aligned}
\label{orth_robustness}
\end{equation}
\end{theorem}

\vspace{-2mm}

See Appendix 7.2 for proof. With this lower bound, we can discuss designing a regulation method to explicitly enhance the certified robustness. 
Firstly, as analyzed in Section 3.3, the imbalance problem brought by modality preference impacts the unreliable modal representation margin. To explicitly regulate the uni-modal encoder and classifier independent from integration, we redefine the margin as the difference in uni-modal score between ground truth and the other label, which can also reflect the reliability of uni-modality. Then we propose to enlarge the relatively small margin through regularization, which is expressed as follows: 
\begin{equation}
\begin{aligned}
        \max_{\tilde{W}^{(m)}, \phi^{(m)}} \min_{m; k \neq y} \quad \tilde{W}^{(m)}_{y\cdot} \phi^{(m)} (\bm{x}^{(m)})- \tilde{W}^{(m)}_{k\cdot} \phi^{(m)} (\bm{x}^{(m)}).   
\end{aligned}
\label{equation_11}
\end{equation}
Based on \autoref{equation_11}, we propose a regularization term, refining the maximum using the LogSumExp function to facilitate better optimization:
\begin{equation}
    L_1 = \frac{1}{N} \sum_{i=1}^N   \log \left(\sum_{m=1}^2 \frac{\sum_{k \neq y} \exp( \tilde{W}^{(m)}_{k\cdot} \phi^{(m)} (\bm{x}_i^{(m)}))}{ \exp(\tilde{W}^{(m)}_{y\cdot} \phi^{(m)} (\bm{x}_i^{(m)}))}\right),
\end{equation}
which is detailed in the Appendix 7.3. Furthermore, we can adjust the integration of different modalities by enlarging the lower bound in \autoref{orth_robustness}. Subsequently, we propose a two-step training procedure called \textit{Certifiable Robust Multi-modal Training} (CRMT), which can credibly obtain a robust multi-modal model. The training procedure of CRMT is as follows:

Step 1: optimize with cross-entropy loss and margin regularization with term $\rho$: \\$\min_{\bm{a}^{(m)}, \tilde{W}^{(m)}, \phi^{(m)}} \rho L_1 + \frac{1}{N} \sum_{i=1}^N CE(h(\bm{x}_i), y_i), $
where $CE$ is the cross-entropy loss function.

Step 2: fix $\tilde{W}^{(m)}, \phi^{(m)}$, update $\bm{a}^{(m)}$ to approach higher certified robustness:\\
$\min_{\bm{a}^{(m)}}~~ L_2 = -\frac{1}{N} \sum_{i=1}^N r(\bm{x}_i),$
where $r(\bm{x})$ is the lower bound in \autoref{orth_robustness}. 

As shown in Appendix 8.3, we further demonstrate that only one iteration of our method can achieve considerable robustness, without a huge consumption of training cost.



\vspace{-2mm}
\begin{table}[t]
\centering
{
\caption{Experimental results of adversarial accuracy results. Bold and underlined values indicate the top and the runner-up results. Combined with JT, MMAT, and Mixup, our proposed CRMT-based methods can improve both performance and robustness.}
\label{main_result}
\vspace{-2mm}
\footnotesize
\begin{tabular}{c|ccc|ccc|ccc}
\toprule
Attack   & \multicolumn{3}{c|}{w\textbackslash{}o}          & \multicolumn{3}{c|}{FGM}                         & \multicolumn{3}{c}{$\ell_2$ PGD}                 \\
Datasets & KS             & UCF101         & VGGS           & KS             & UCF101         & VGGS           & KS             & UCF101         & VGGS           \\ \midrule
JT       & 0.643          & 0.742          & 0.496          & 0.288          & 0.506          & 0.157          & 0.268          & 0.431          & 0.056          \\ \midrule
GB       & 0.704          & {\ul 0.784}    & 0.529          & 0.371          & 0.432          & 0.279          & 0.344          & 0.205          & 0.174          \\
OGM      & 0.651          & 0.743          & 0.498          & 0.292          & 0.492          & 0.190          & 0.270          & 0.377          & 0.069          \\
PMR      & 0.689          & 0.742          & 0.503          & 0.350          & 0.455          & 0.166          & 0.331          & 0.292          & 0.059          \\ \midrule
MSEFM    & 0.627          & 0.721          & 0.492          & 0.390          & 0.483          & 0.187          & 0.376          & 0.243          & 0.115          \\
MMAT     & 0.656          & 0.728          & 0.509          & {\ul 0.514}    & {\ul 0.609}    & {\ul 0.413}    & {\ul 0.507}    & {\ul 0.598}    & {\ul 0.409}    \\
Mixup    & 0.669          & 0.717          & 0.507          & 0.347          & 0.413          & 0.278          & 0.327          & 0.192          & 0.135          \\ \midrule
CMRT     & {\ul 0.758}    & \textbf{0.789} & 0.526          & 0.491          & 0.515          & 0.248          & 0.468          & 0.433          & 0.124          \\
CMRT-AT  & \textbf{0.762} & 0.759          & \textbf{0.538} & \textbf{0.608} & \textbf{0.614} & \textbf{0.422} & \textbf{0.602} & \textbf{0.602} & \textbf{0.414} \\
CMRT-Mix & 0.744          & 0.769          & {\ul 0.537}    & {\ul 0.514}    & 0.430          & 0.328          & 0.491          & 0.191          & 0.184          \\ \bottomrule
\end{tabular}
}
\vspace{-1mm}
\end{table}
\section{Experiments}
\vspace{-1mm}
\subsection{Setups}
\vspace{-1mm}
\paragraph{Dataset.}
We evaluate our method on different datasets including Kinetics-Sounds (Audio + Vision) \citep{arandjelovic2017look}, UCF101 (Optical flow + RGB) ~\citep{soomro2012ucf101}, and VGGSound (Audio + Vision) ~\citep{chen2020vggsound}. We use the backbone ResNet18~\citep{he2016deep} as the encoder for each uni-modality. Details about these datasets are presented in Appendix 8.1.  

\vspace{-1mm}

\paragraph{Multi-modal models.}

In this study, the comparison methods are selected to improve multi-modal learning and be suitable for the multi-modal joint training strategy. Comparison methods can be divided into two distinct groups:
Firstly, methods address the imbalance problem caused by modality preference: 
\textit{Gradient Blending} (GB)~\citep{wang2020makes}, \textit{On-the-fly Gradient Modulation with generalization enhancement} (OGM)~\citep{peng2022balanced},  \textit{Prototypical Modal Rebalance} (PMR)~\citep{fan2023pmr}. 
Secondly, methods aim at improving multi-modal robustness: 
\textit{Multi-Modal Adversarial Training} (MMAT)~\citep{li2022adversarial}, \textit{Multi-modal mixup} (Mixup)~\citep{madry2017towards, li2022adversarial}, \textit{MinSim+ExFMem} (MSEFM)~\citep{tian2021can}.  Our method can be extended to different training strategies, denoted as \textit{Certifiable Robust Multi-modal Training with Joint Training} (CRMT-JT), CRMT \textit{with Adversarial Training} (CRMT-AT), and CRMT \textit{with Mixup} (CRMT-Mix).

\vspace{-1mm}

\paragraph{Attack methods.}
Following previous work~\citep{tsuzuku2018lipschitz, singla2021improved}, we select \textit{Fast Gradient Method} (FGM)~\citep{goodfellow2014explaining} and  \textit{$\ell_2$ Projected Gradient Descent} ($\ell_2$-PGD)~\citep{madry2017towards} as two attack methods with attack size $\epsilon =  0.5$, which are widely used for verifying multi-modal robustness with $\ell_2$-norm. For uni-modal attack, we also introduce attacks FGM and $\ell_{2}$-PGD with $\epsilon =  1.0$, and missing on uni-modality.  

\vspace{-2mm}
\subsection{Robustness validation on multi-modal and uni-modal attack}
\vspace{-1mm}
\paragraph{Robustness against multi-modal attacks.} We validate the robustness of our method under multi-modal attacks. Based on the experimental results presented in \autoref{main_result}, we have identified four key observations that warrant attention.
Firstly, while imbalance methods effectively enhance performance on clean samples, robustness methods can demonstrate more notable defense capabilities against various attacks. Secondly, the imbalance method GB can be superior to the robustness method MSEFM under some situations. That is because GB introduces a multi-task method to enhance the learning of uni-modality and improve the uni-modal margin, thus it can enhance the robustness methods according to our analysis.
Thirdly, our proposed CRMT-based methods surpass the performance of the compared methods across these datasets in most cases. This superiority stems from our approach to addressing the imbalance problem through improving uni-modal representation margins, and the certificate adjustment of modality-specific weights for heightened certified robustness. Fourthly, the improvement of results for CRMT-AT and CRMT-Mix suggests that our training procedure can serve as a valuable component applicable to other robust training methods.

\vspace{-2mm}
\begin{table}[t]
\vspace{-2mm}
\centering
{ 
 \caption{Performance against distinct uni-modal attack methods on KS dataset. }
 \label{distinct_1}
\begin{tabular}{c|c|c|cc|cc|cc}
\toprule
\multirow{2}{*}{Method}     & \multirow{2}{*}{Attack} & w\textbackslash{}o & \multicolumn{2}{c|}{ FGM}  & \multicolumn{2}{c|}{ $\ell_2$-PGD} & \multicolumn{2}{c}{Missing modality} \\
                            &                         & Clean              & \#v            & \#a            & \#v                & \#a                & \#v               & \#a              \\ \midrule
Baseline                    & JT                      & 0.643              & 0.616          & 0.143          & 0.616              & 0.110              & 0.480             & 0.230            \\ \midrule
\multirow{3}{*}{Imbalance}  & GB                      & 0.704              & 0.658          & 0.195          & 0.656              & 0.157              & 0.483             & 0.430            \\
                            & OGM                     & 0.651              & 0.624          & 0.147          & 0.624              & 0.116              & 0.463             & 0.207            \\
                            & PMR                     & 0.689              & 0.649          & 0.181          & 0.649              & 0.152              & 0.484             & 0.315            \\ \midrule
\multirow{3}{*}{Robustness} & MSEFM                   & 0.627              & 0.577          & 0.320          & 0.576              & 0.314              & 0.466             & 0.239            \\
                            & MMAT                    & 0.656              & 0.633          & {\ul 0.390}    & 0.632              & {\ul 0.369}        & 0.482             & 0.262            \\
                            & Mixup                   & 0.669              & 0.628          & 0.191          & 0.626              & 0.147              & 0.450             & 0.256            \\ \midrule
\multirow{3}{*}{Ours}       & CRMT-JT                 & {\ul 0.758}        & 0.685          & 0.384          & 0.682              & 0.327              & {\ul 0.560}       & {\ul 0.591}      \\
                            & CRMT-AT                 & \textbf{0.762}     & {\ul 0.703}    & \textbf{0.488} & {\ul 0.698}        & \textbf{0.464}     & 0.547             & \textbf{0.626}   \\
                            & CRMT-Mix                & 0.744              & \textbf{0.706} & 0.377          & \textbf{0.704}     & 0.320              & \textbf{0.572}    & 0.566            \\ \bottomrule
\end{tabular}
}
\vspace{-4mm}
\end{table}
\vspace{-1mm}

\begin{figure}[b]
\vspace{-3mm}
\centering
\subcaptionbox{$\eta^{(v)}/\eta^{(a)}$ in MMAT. }{
\includegraphics[width=0.31\textwidth]{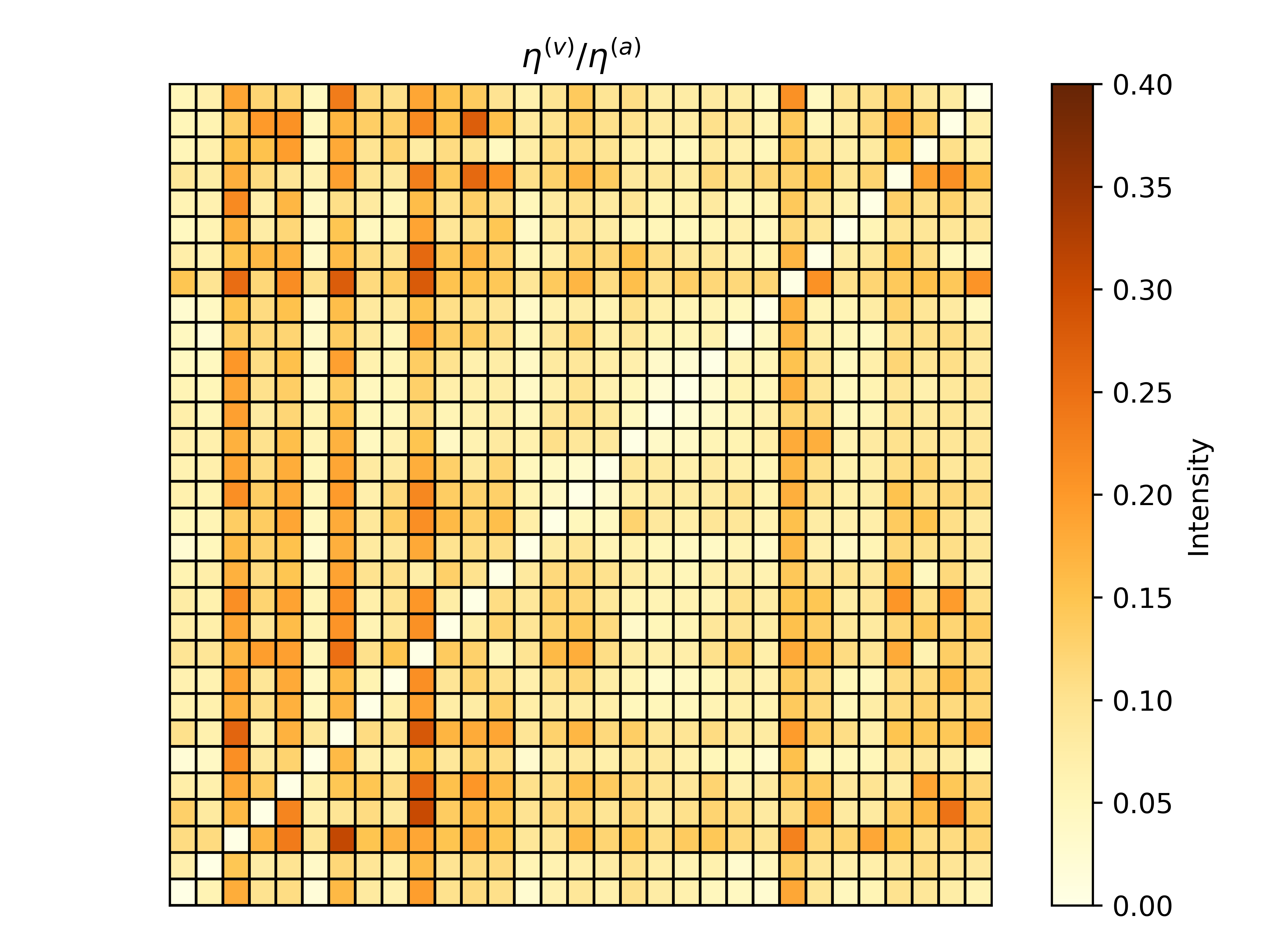}}
\subcaptionbox{$\eta^{(v)}/\eta^{(a)}$ in CRMT-AT step-1.}{
\includegraphics[width=0.31\textwidth]{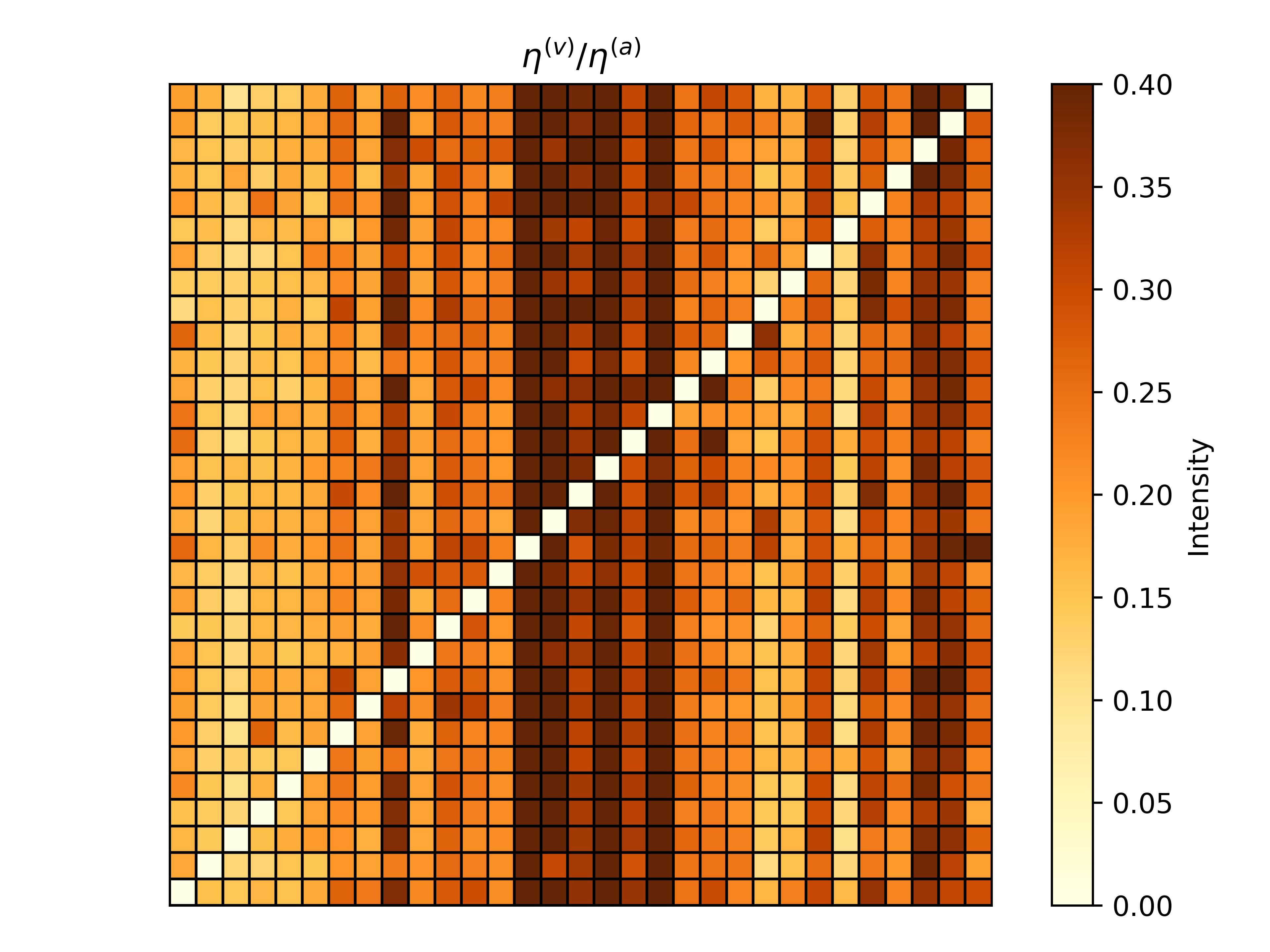}}
\subcaptionbox{$\eta^{(v)}/\eta^{(a)}$ in CRMT-AT.}{
\includegraphics[width=0.31\textwidth]{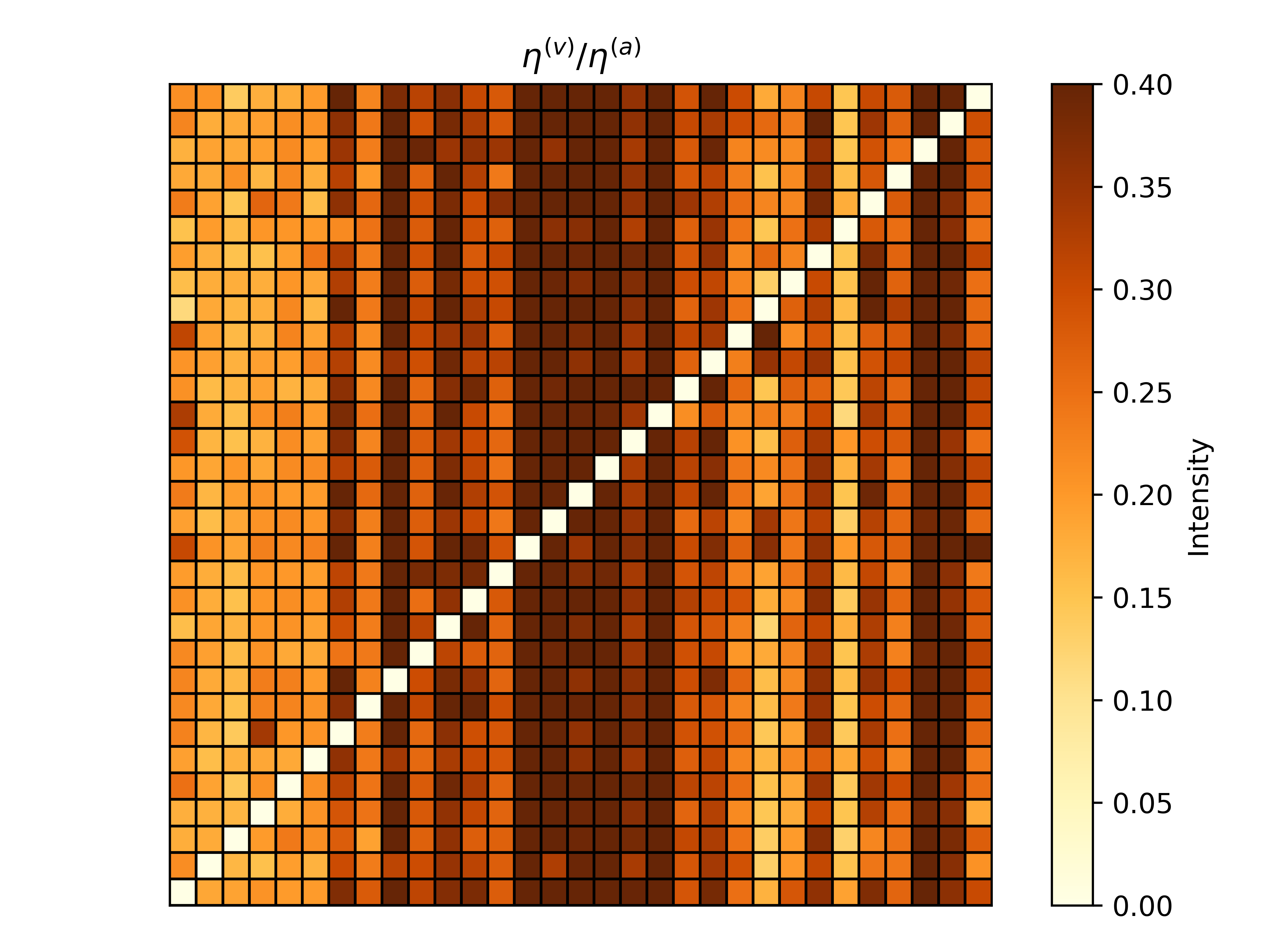}}
\vspace{-2mm}
\caption{Evaluation of the ratio of vulnerability indicators between modality \#$v$ and \#$a$ (preferred). We illustrate the ratio in MMAT, CRMT-AT, and CRMT-AT with only the first training procedure. 
}
\label{preference}
\vspace{-1mm}
\end{figure}

\paragraph{Robustness against distinct uni-modal attack.}

To substantiate the efficacy of our CRMT methods, we conducted additional experiments on the KS dataset, involving a broader range of different uni-modal attacks. As demonstrated in \autoref{distinct_1}, the absence of modality \#$a$ leads to a larger performance decline than \#$v$, since it is more preferred by multi-modal models. 
As shown in \autoref{distinct_1}, previous multi-modal methods demonstrated poor performance when subjected to attacks on the preferred modality \#$a$, aligning with our analysis. In contrast, our CRMT-based methods consider addressing the imbalance problem introduced by modality preference and adjusting integration, thus making the model more robust against uni-modal attacks.
Furthermore, when encountering different uni-modal attacks, our CRMT-based approach shows superior performance and consistently ensures higher robustness in various scenarios. This strongly demonstrates the effectiveness and versatility of our proposed method in enhancing the robustness of multi-modal models.

\vspace{-3mm}

\paragraph{Imbalance on uni-modal vulnerability indicators.}
\begin{wrapfigure}{rh!}{0.35\linewidth}
\vspace{-5mm}
\includegraphics[width=\linewidth]{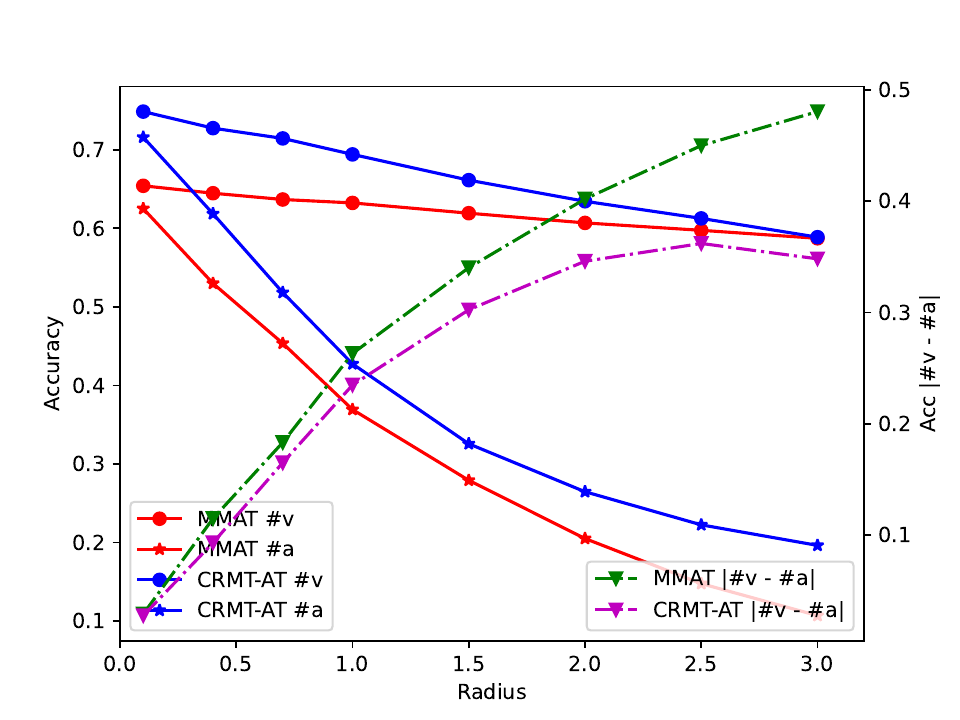}
\vspace{-4mm}
\caption{This figure presents the robustness accuracy against uni-modal attacks with different sizes, where the dotted line signifies the difference in robustness accuracy between two uni-modalities. }
  \label{figure_4}
\vspace{-5mm}
\end{wrapfigure}
We verify the conclusion drawn in Section 3.3 that the imbalance problem introduced by modality preference leads to the vulnerability of attack on specific modality, as illustrated in \autoref{teasor}. To achieve this, we compare MMAT with our CRMT-AT approach. As depicted in \autoref{uni-modal-perturbation}, $\eta^{(m)}$ is utilized as a measurement to assess the robustness of multi-modal models against uni-modal perturbations, and we employ the ratio of uni-modal vulnerability indicators $\eta^{(v)}/\eta^{(a)}$ (see ~\autoref{preference}). It can be demonstrated that our approach diligently works to reduce the imbalance in the indicator $\eta$, effectively darkening the heated map. Furthermore, according to the robustness accuracy, our approach significantly reduces the disparity in attack performance between the two modalities (see ~\autoref{figure_4}). This evidence serves to illustrate the efficacy of our method in mitigating the imbalance problem and ultimately enhancing multi-modal robustness.

\vspace{-2mm}

\subsection{Validation for Effectiveness and scalability}

\vspace{-1mm}
\begin{wrapfigure}{rhtp!}{0.32\linewidth}
\vspace{-5mm}
\includegraphics[width=\linewidth]{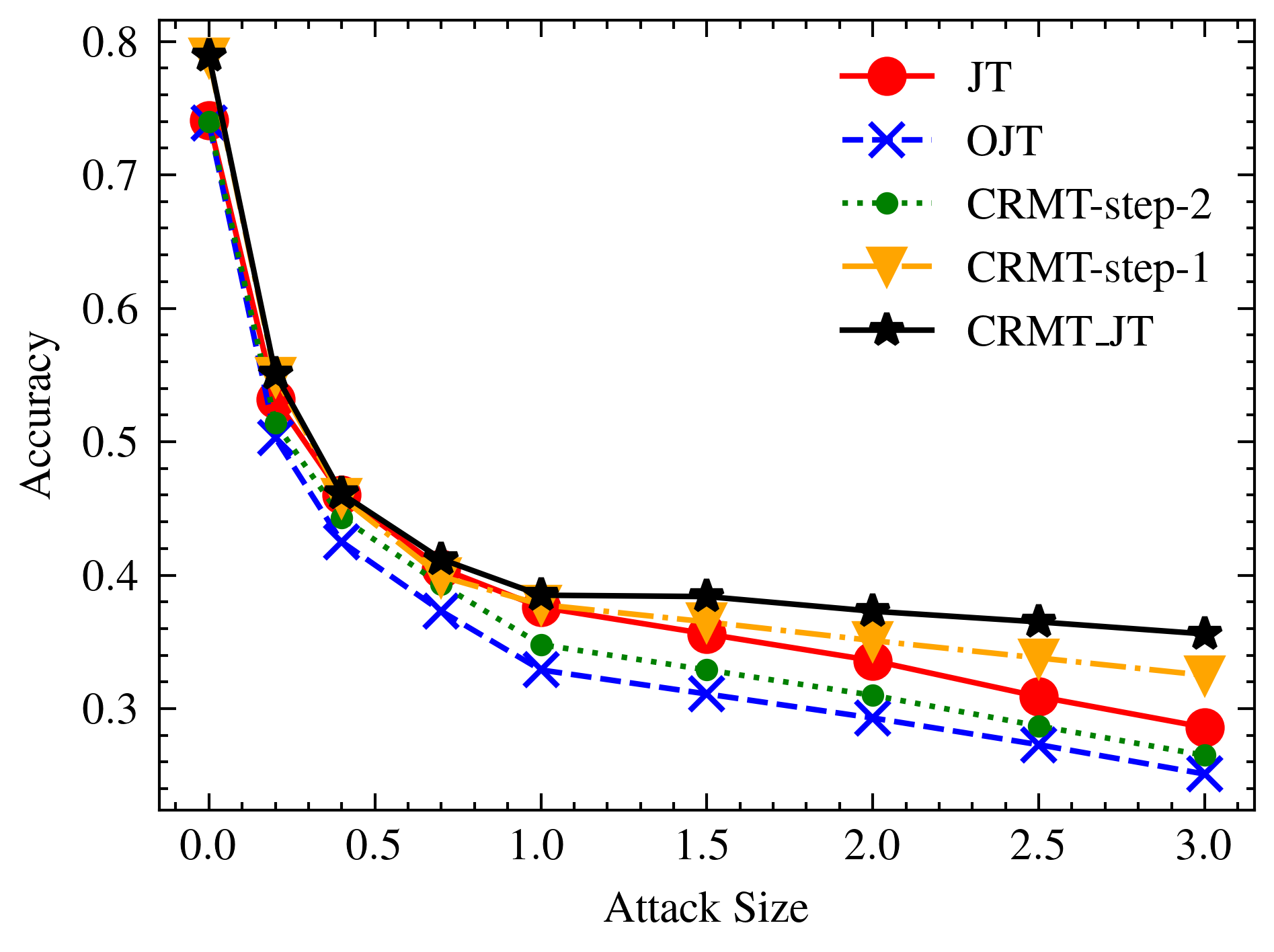}
\caption{Ablation studies of our methods on the UCF101 dataset, revealing the effect of each part we introduced. }
  \label{ablation_studies}
\vspace{-3mm}
\end{wrapfigure}
\paragraph{Ablation studies. }
To delve deeper into the efficacy of our training procedure, we conduct an ablation analysis, aiming to elucidate the contribution of each component of the proposed training procedure toward the overall result. When compared with the baseline Joint Training (JT), our results include the orthogonal-based framework and our proposed training procedure. We report our findings related to solely using the orthogonal framework (OJT), CRMT focusing exclusively on each step (CRMT-step-1 and CRMT-step-2), and our full proposal, CRMT-JT, which incorporates both steps. As depicted in \autoref{ablation_studies}, it is evident that the orthogonal-based framework does not contribute to improved robustness. Meanwhile, both step-1 and step-2 of the training procedure can enhance robustness, particularly in the face of larger-size attacks. Furthermore, the results highlight that each step consistently enhances multi-modal robustness, making CRMT-JT superior in multi-modal robustness.

\vspace{-2mm}
\begin{table}[t]
\centering
{
\caption{Extension results of adversarial accuracy with transformer-based method on VGGS dataset.}
\label{extension}
\vspace{-2mm}
\begin{tabular}{c|c|cccc|cc}
\toprule
Attack   & Clean              & \multicolumn{4}{c|}{Uni-modal attack}                   & \multicolumn{2}{c}{Multi-modal attack} \\ \midrule
Method   & w\textbackslash{}o & FGM \#v & FGM \#a & $\ell_2$-PGD \#v & $\ell_2$-PGD \#a & FGM            & $\ell_2$-PGD          \\ \midrule
MMT      & 0.465              & 0.433   & 0.259   & 0.431            & 0.228            & 0.331          & 0.326                 \\
CRMT-MMT & 0.471              & 0.440   & 0.276   & 0.428            & 0.237            & 0.344          & 0.340                 \\ 
\bottomrule
\end{tabular}
}
\vspace{-4mm}
\end{table}
\paragraph{Extension studies to transformer.}
\vspace{-2mm}
We further provide experiment results of our methods extended to the Multi-Modal Transformer-based framework with hierarchical attention~\citep{xu2023multimodal} (MMT) on the VGGS dataset.
Both visual and audio class tokens are concatenated and linearly projected into the output space. We mainly evaluate the robustness under the Transformer model accurately, hence we train from scratch instead of using the pre-trained model. To validate our method, we only introduce the CRMT procedure step-1 on this transformer architecture as a comparison. That is because both audio and visual input influence each token, which is beyond our assumption. 
Based on the results in~\autoref{extension}, our method can also improve the robustness of transformer-based architecture in most cases, which proves the broad prospects of our approach.

\vspace{-3mm}

\vspace{-1mm}
\section{Conclusion}
\vspace{-3mm}
In this study, we present the essential components for multi-modal robustness and delve into the limitations imposed by modality preference. 
Additionally, we explore why multi-modal models exhibit vulnerabilities to attack specific modalities. Further, we introduce the Certifiable Robust Multi-modal Training procedure, a novel approach explicitly designed to boost the certified robustness. 

\section{Acknowledgement}
This research was supported by National Natural Science Foundation of China (NO.62106272), the Young Elite Scientists Sponsorship Program by CAST (2021QNRC001), and Public Computing Cloud, Renmin University of China.

\bibliography{iclr2024_conference}
\bibliographystyle{iclr2024_conference}
\newpage

\section{DETAILED ANALYSIS AND PROOFS IN SECTION 3}

\subsection{Proof of Theorem 1}
\label{Proof_theorem_1}
In this section, we provide proof for two lower bounds in \autoref{ori_robustness} and \autoref{orth_robustness}. 

\begin{theorem}
     Given an input $\bm{x} = (\bm{x}^{(1)},\bm{x}^{(2)})$ with ground-truth label $y \in [K]$ and the runner-up label $j \neq y$, $\zeta_{j}^{(m)}(\bm{x}^{(m)})$ as the representation margin for $m$-th modality, satisfying the Lipschitz constraint $\tau_j^{(m)}$, and the fusion factor $c^{(m)}_j$. Define $\bm{x}'=(\bm{x}’^{(1)},\bm{x}‘^{(2)})$ as the perturbed sample, and $\bm{x} - \bm{x}'$ denotes the perturbation. The bound for the perturbation radius can be described as:
\begin{equation}
\begin{aligned}
 P(\bm{x}) = \min_{\bm{x}'} \left\|\bm{x} - \bm{x}'\right\|_2 
& \geq \frac{c_j^{(1)} \zeta_{j}^{(1)}(\bm{x}^{(1)}) + c_j^{(2)} \zeta_{j}^{(2)}(\bm{x}^{(2)})+ \beta_j}{\sqrt{(c_j^{(1)} \tau_j^{(1)})^2 +(c_j^{(2)} \tau_j^{(2)})^2 }} \\
  s.t.  ~~ c_j^{(1)} \zeta_{j}^{(1)}(\bm{x}'^{(1)}) &+ c_j^{(2)} \zeta_{j}^{(2)}(\bm{x}'^{(2)}) + \beta_j = 0.
\end{aligned}
\label{ori_robustness_proof}
\end{equation}
\end{theorem}

\begin{proof}
The key to this proof is how to bridge the gap between multi-modal perturbation and the uni-modal representation, by properly introducing the perturbed sample through the condition in Inequality \ref{ori_robustness_proof} and the margin on perturbation $\zeta^{(m)}_j(\bm{x}'^{(m)})$. 
First, considering this constraint, we can obtain the following equation:

\begin{equation}
\begin{aligned}
       h_y(\bm{x}) - h_j(\bm{x}) &= c_j^{(1)} \zeta_{j}^{(1)}(\bm{x}^{(1)}) + c_j^{(2)} \zeta_{j}^{(2)}(\bm{x}^{(2)})+ \beta_j \\ 
       &= c_j^{(1)} (\zeta_{j}^{(1)}(\bm{x}^{(1)}) - \zeta_{j}^{(1)}(\bm{x}'^{(1)})) + c_j^{(2)} (\zeta_{j}^{(2)}(\bm{x}^{(2)}) - \zeta_{j}^{(2)}(\bm{x}'^{(2)})),
\end{aligned}
\label{uni-modal-perturbation_1}
\end{equation}
where the formulation of the difference of margin occurs, which can serve as a bond to connect with the perturbation. Then consider the following inequality:
\begin{equation}
\begin{aligned}
       &c_j^{(1)} |\zeta_{j}^{(1)}(\bm{x}^{(1)}) - \zeta_{j}^{(1)}(\bm{x}'^{(1)})|+ c_j^{(2)} |\zeta_{j}^{(2)}(\bm{x}^{(2)}) - \zeta_{j}^{(2)}(\bm{x}'^{(2)})| \\ 
       & \leq  c_j^{(1)} \tau_j^{(1)}\left\|\bm{x}^{(1)} - \bm{x}'^{(1)}\right\|_2 + c_j^{(2)} \tau_j^{(2)}\left\|\bm{x}^{(2)} - \bm{x}'^{(2)}\right\|_2 . 
\end{aligned}
\label{forward_1}
\end{equation}

Once the condition $\forall m = 1, 2, \zeta_{j}^{(m)}(\bm{x}^{(m)}) \geq \zeta_{j}^{(m)}(\bm{x}'^{(m)})$ is satisfied, which imposes the minimum requirement on $\bm{x}'$, as proven later, we can establish the relationship with the perturbation.
Furthermore, by employing the Cauchy Inequality, we obtain the following inequality:

\begin{equation}
\begin{aligned}
      c_j^{(1)} \tau_j^{(1)}\left\|\bm{x}^{(1)} - \bm{x}'^{(1)}\right\|_2 + c_j^{(2)} \tau_j^{(2)}\left\|\bm{x}^{(2)} - \bm{x}'^{(2)}\right\|_2 \leq  \sqrt{(c_j^{(1)} \tau_j^{(1)})^2 +(c_j^{(2)} \tau_j^{(2)})^2 } \left\|\bm{x} - \bm{x}'\right\|_2 . 
\end{aligned}
\label{forward_2}
\end{equation}

Thus, we have successfully established the bound for perturbation. However, two challenges remain unresolved. Firstly, we need to minimize the objective function by finding an optimal perturbation $\bm{x}'$. Secondly, the validity of condition $\zeta_{j}^{(m)}(\bm{x}^{(m)}) \geq \zeta_{j}^{(m)}(\bm{x}'^{(m)})$ cannot be guaranteed. In the subsequent sections, we will address these issues and provide formal proofs to support our claims.

Initially, drawing inspiration from the Inequality \ref{uni-modal-perturbation}-\ref{forward_2}, we perform a straightforward transformation on the Inequality \ref{ori_robustness_proof}, aiming to establish a connection between multi-modal perturbation and uni-modal perturbation:

\begin{equation}
\begin{aligned}
 \min_{\bm{x}'} \sqrt{(c_j^{(1)} \tau_j^{(1)})^2 +(c_j^{(2)} \tau_j^{(2)})^2 } \left\|\bm{x} - \bm{x}'\right\|_2 
& \geq c_j^{(1)} \zeta_{j}^{(1)}(\bm{x}^{(1)}) + c_j^{(2)} \zeta_{j}^{(2)}(\bm{x}^{(2)})+ \beta_j 
 \\
  s.t.  ~~ c_j^{(1)} \zeta_{j}^{(1)}(\bm{x}'^{(1)}) &+ c_j^{(2)} \zeta_{j}^{(2)}(\bm{x}'^{(2)}) + \beta_j = 0.
\end{aligned}
\label{inequation_1}
\end{equation}
Subsequently, we arrive at the following inequality: 

    \begin{equation}
\begin{aligned}
&\min_{\bm{x}'} \sqrt{(c_j^{(1)} \tau_j^{(1)})^2 +(c_j^{(2)} \tau_j^{(2)})^2 } \left\|\bm{x} - \bm{x}'\right\|_2 \\
\geq & \min_{\bm{x}'^{(1)}, \bm{x}'^{(2)}} c_j^{(1)} \tau_j^{(1)}\left\|\bm{x}^{(1)} - \bm{x}'^{(1)}\right\|_2 + c_j^{(2)} \tau_j^{(2)}\left\|\bm{x}^{(2)} - \bm{x}'^{(2)}\right\|_2  \quad\quad\quad\quad ~(\rm{Cauchy ~Inequality}) \\
\geq & \min_{\bm{x}'^{(1)}, \bm{x}'^{(2)}} c_j^{(1)} |\zeta_{j}^{(1)}(\bm{x}^{(1)}) -\zeta_{j}^{(1)}(\bm{x}'^{(1)})| + c_j^{(2)} |\zeta_{j}^{(2)}(\bm{x}^{(2)}) -\zeta_{j}^{(2)}(\bm{x}'^{(2)})| ~(\rm{Lipschitz ~condition})
\\ = & \quad  c_j^{(1)} \zeta_{j}^{(1)}(\bm{x}^{(1)}) + c_j^{(2)} \zeta_{j}^{(2)}(\bm{x}^{(2)})+ \beta_j,  \quad \quad \quad \quad\quad \quad\quad\quad \quad \quad  \quad\quad \quad\quad\quad \quad \quad\quad\quad \quad \rm{(*)}
\end{aligned}
\label{robustness_main}
\end{equation}
where the Equation $(*)$ is required to be proved. 
In the following, we prove this target equation in detail. The main concern is that whether $\zeta_{j}^{(m)}(\bm{x}^{(m)}) \geq \zeta_{j}^{(m)}(\bm{x}'^{(m)})$ holds for each modality $m$.
As the robustness analysis focuses solely on correctly classified samples, where $\forall j \neq y, h_y(\bm{x})>h_j(\bm{x})$, we can establish the following inequality:

\begin{equation}
    c_j^{(1)} \zeta_{j}^{(1)}(\bm{x}^{(1)}) + c_j^{(2)} \zeta_{j}^{(2)}(\bm{x}^{(2)}) + \beta_j > 0.
    \label{perturbation_0}
\end{equation} 

The condition stated in Inequality \ref{ori_robustness_proof} can be expressed as follows:

\begin{equation}
   h_y(\bm{x}')-h_j(\bm{x}') =  c_j^{(1)} \zeta_{j}^{(1)}(\bm{x}'^{(1)}) + c_j^{(2)} \zeta_{j}^{(2)}(\bm{x}'^{(2)}) + \beta_j = 0.
    \label{condition_1}
\end{equation}
By incorporating Inequality \ref{perturbation_0} and \autoref{condition_1}, we can obtain the following inequality:

\begin{equation}
    c_j^{(1)} \zeta_{j}^{(1)}(\bm{x}^{(1)}) + c_j^{(2)} \zeta_{j}^{(2)}(\bm{x}^{(2)}) > c_j^{(1)} \zeta_{j}^{(1)}(\bm{x}'^{(1)}) + c_j^{(2)} \zeta_{j}^{(2)}(\bm{x}'^{(2)}).
    \label{perturbation_1}
\end{equation} 

Without loss of generality, we suppose $\zeta_{j}^{(1)}(\bm{x}^{(1)}) < \zeta_{j}^{(1)}(\bm{x}'^{(1)})$, and we have:

\begin{equation}
\begin{aligned}
c_j^{(1)} \zeta_{j}^{(1)}(\bm{x}^{(1)}) + c_j^{(2)} \zeta_{j}^{(2)}(\bm{x}'^{(2)}) + \beta_j < 0. 
\end{aligned}
\label{perturbation_2}
\end{equation}
Hence, since the margin function $\zeta_{j}^{(2)}$ is continuous, we define an additional variable to describe this process. Let $f(\lambda) = c_j^{(1)} \zeta_{j}^{(1)}(\bm{x}^{(1)}) + c_j^{(2)} \zeta_{j}^{(2)}(\lambda\bm{x}^{(2)} + (1 - \lambda)\bm{x}'^{(2)}) + \beta_j$, we can observe that both $f(0)$ and $f(1)$ equal zero. 
Applying the Zero Existence Theorem, we can always find a $0 <\hat{\lambda} < 1$, and a perturbation $\hat{\bm{x}}'^{(2)} = \hat{\lambda}\bm{x}^{(2)} + (1 - \hat{\lambda})\bm{x}'^{(2)}$.
satisfying: 
\begin{equation}
c_j^{(1)} \zeta_{j}^{(1)}(\bm{x}^{(1)}) + c_j^{(2)} \zeta_{j}^{(2)}(\hat{\bm{x}}'^{(2)}) + \beta_j =0.
\end{equation}
Hence, $\hat{\bm{x}} = (\bm{x}^{(1)},\hat{\bm{x}}'^{(2)})$ is also a feasible perturbation satisfying the condition of \autoref{ori_robustness_proof}. And obviously, 
\begin{equation}
    \|\bm{x} - \hat{\bm{x}}\|_2 = \|(1 - \hat{\lambda}) (\bm{x}^{(2)} - \bm{x}'^{(2)})\|_2 \leq \|\bm{x} - \bm{x}'\|_2,
\end{equation}
which contradicts the minimization constraint. 

In conclusion, we have verify that $\zeta_{j}^{(m)}(\bm{x}^{(m)}) \geq \zeta_{j}^{(m)}(\bm{x}'^{(m)})$ for each modality $m$. Thus, we can approach the target equation in \autoref{robustness_main} by the following equation: 
\begin{equation}
\begin{aligned}
        &\quad \min_{\bm{x}'^{(1)}, \bm{x}'^{(2)}} c_j^{(1)} |\zeta_{j}^{(1)}(\bm{x}^{(1)}) -\zeta_{j}^{(1)}(\bm{x}'^{(1)})| + c_j^{(2)} |\zeta_{j}^{(2)}(\bm{x}^{(2)}) -\zeta_{j}^{(2)}(\bm{x}'^{(2)})| 
    \\ &=  \min_{\bm{x}'^{(1)}, \bm{x}'^{(2)}} c_j^{(1)} (\zeta_{j}^{(1)}(\bm{x}^{(1)}) -\zeta_{j}^{(1)}(\bm{x}'^{(1)})) + c_j^{(2)} (\zeta_{j}^{(2)}(\bm{x}^{(2)}) -\zeta_{j}^{(2)}(\bm{x}'^{(2)}))
    \\ &=  c_j^{(1)} \zeta_{j}^{(1)}(\bm{x}^{(1)}) + c_j^{(2)} \zeta_{j}^{(2)}(\bm{x}^{(2)})+ \beta_j. 
\end{aligned}
\end{equation}

Therefore, combined with \autoref{robustness_main}, we can obtain the target inequality. 
\begin{equation}
    \min_{\bm{x}'} \left\|\bm{x} - \bm{x}'\right\|_2  \geq \frac{c_j^{(1)} \zeta_{j}^{(1)}(\bm{x}^{(1)}) + c_j^{(2)} \zeta_{j}^{(2)}(\bm{x}^{(2)})+ \beta_j}{\sqrt{(c_j^{(1)} \tau_j^{(1)})^2 +(c_j^{(2)} \tau_j^{(2)})^2 }}.
\end{equation}
\end{proof}

\subsection{Proof of Theorem 2}
\label{Proof_theorem_2}
\begin{theorem} 
Given an input $\bm{x} =  (\bm{x}^{(1)},\bm{x}^{(2)})$ with ground-truth label $y \in [K]$ and the runner-up label $j \neq y$, the orthogonal classifier $\tilde{W}^{(m)}$, the modality-specific weight $\bm{a}^{(m)}$, the Lipschitz constant $\tilde{\tau}_j^{(m)}$, and the difference of the bias $\tilde{\beta}_j =  \tilde{b}_y -  \tilde{b}_j$. The bound for the perturbation radius can be described as:
\begin{equation}
\begin{aligned}
    P(\bm{x}) \geq &\frac{\sum_{m=1}^2 \left( a_y^{(m)}\tilde{W}^{(m)}_{y.} \phi^{(m)} (\bm{x}^{(m)})- a_j^{(m)}\tilde{W}^{(m)}_{j.} \phi^{(m)} (\bm{x}^{(m)})\right) + \tilde{\beta}_j  }{\sqrt{\sum_{m=1}^2 \left( a^{(m)}_y \tilde{\tau}^{(m)}_y + a^{(m)}_j \tilde{\tau}^{(m)}_j\right)^2  }}  \\
  &\quad \quad \quad \quad s.t.  ~~ \exists j \neq y, ~ \tilde{h}_y(\bm{x}') = \tilde{h}_j(\bm{x}').
\end{aligned}
\label{orth_robustness_1}
\end{equation}
\end{theorem}

\begin{proof}

For simplify, here we use $\tilde{h}_y^{(m)}(\bm{x}^{(m)})= a_y^{(m)}\tilde{W}^{(m)}_{y.} \phi^{(m)} (\bm{x}^{(m)})$ as the score of the sample $\bm{x}$ of $y$-th class in $m$-th modality. Hence, we have:
\begin{equation}
    \tilde{h}_y = \tilde{h}_y^{(1)}(\bm{x}^{(1)}) + \tilde{h}_y^{(2)}(\bm{x}^{(2)}) + b_y. 
\end{equation}

Notice that the numerator term of the lower bound equals to $\tilde{h}_y(\bm{x}) - \tilde{h}_j(\bm{x})$, combining with the condition, we have:

\begin{equation}
 \begin{aligned}
     \tilde{h}_y(\bm{x}) - \tilde{h}_j(\bm{x}) = \tilde{h}_y(\bm{x}) - \tilde{h}_j(\bm{x})  - (\tilde{h}_y(\bm{x}') - \tilde{h}_j(\bm{x}'))\\
 = \sum_{m=1}^2 \tilde{h}^{(m)}_y(\bm{x}^{(m)}) - \tilde{h}^{(m)}_j(\bm{x}^{(m)}) - (\tilde{h}^{(m)}_y(\bm{x}'^{(m)}) - \tilde{h}^{(m)}_j(\bm{x}'^{(m)})) .
 \end{aligned}
\end{equation}
Denote $\gamma_j^{(m)} (\bm{x}^{(m)})= \tilde{h}^{(m)}_y(\bm{x}^{(m)}) - \tilde{h}^{(m)}_j(\bm{x}^{(m)})$. Inspired by the proof above, we have :
\begin{equation}
    \gamma_j^{(m)} (\bm{x}^{(m)}) \geq \gamma_j^{(m)} (\bm{x}'^{(m)})
    \label{conditions_for_orth}
\end{equation}
holds for each modality $m$, which requires the minimization condition of $\bm{x}'$ and will be proved later.

Thus, we can use the Lipschitz condition for $j$-th class:
\begin{equation}
   | \tilde{W}^{(m)}_{j.} \phi^{(m)} (\bm{x}^{(m)})- \tilde{W}^{(m)}_{j.} \phi^{(m)} (\bm{x}'^{(m)}) |\leq \tilde{\tau}_j^{(m)} \left\|\bm{x}^{(m)} - \bm{x}'^{(m)}\right\|_2.
\end{equation}
and we have: 

\begin{equation}
\begin{aligned}   
 &\quad ~~ \sum_{m=1}^2 |\gamma_j^{(m)} (\bm{x}^{(m)}) - \gamma_j^{(m)} (\bm{x}'^{(m)})|\\
 &= \sum_{m=1}^2 \left| \tilde{h}^{(m)}_{y}(\bm{x}^{(m)})- \tilde{h}^{(m)}_{y} (\bm{x}'^{(m)}) - \tilde{h}^{(m)}_{j} (\bm{x}^{(m)}) + \tilde{h}^{(m)}_{j} (\bm{x}'^{(m)})\right| \\
 &\leq \sum_{m=1}^2 \left| a_y^{(m)}\tilde{W}^{(m)}_{y.} \phi^{(m)} (\bm{x}^{(m)})- a_y^{(m)}\tilde{W}^{(m)}_{y.} \phi^{(m)} (\bm{x}'^{(m)})\right| \\
 &\quad\quad\quad  + \left| a_j^{(m)}\tilde{W}^{(m)}_{j.} \phi^{(m)} (\bm{x}^{(m)})- a_j^{(m)}\tilde{W}^{(m)}_{j.} \phi^{(m)} (\bm{x}'^{(m)})\right|\\
 &\leq  \left( a^{(1)}_y \tilde{\tau}^{(1)}_y + a^{(1)}_j \tilde{\tau}^{(1)}_j \right)\|\bm{x}^{(1)} - \bm{x}'^{(1)} \|_2 +  \left( a^{(2)}_y \tilde{\tau}^{(2)}_y + a^{(2)}_j \tilde{\tau}^{(2)}_j \right)\|\bm{x}^{(2)} - \bm{x}'^{(2)} \|_2
\end{aligned}
\end{equation}

Thus, we can further solve it by employing the Cauchy Inequality. 

\begin{equation}
     \sum_{m=1}^2 |\gamma_j^{(m)} (\bm{x}^{(m)}) - \gamma_j^{(m)} (\bm{x}'^{(m)})| \leq \sqrt{\sum_{m=1}^2 \left( a^{(m)}_y \tilde{\tau}^{(m)}_y + a^{(m)}_j \tilde{\tau}^{(m)}_j\right)^2  } \|\bm{x} - \bm{x}'\|_2. 
\end{equation}

Further, we verify that when $\bm{x}'$ sets the size of perturbation the minimum value, \autoref{conditions_for_orth} satisfies. The proof is similar to the proof above. 
As our robustness analysis focuses solely on correctly classified samples, we can establish the following inequality:
\begin{equation}
  \forall j \neq y,  \tilde{h}_y (\bm{x}) -  \tilde{h}_j (\bm{x}) = \gamma_j^{(1)} (\bm{x}^{(1)}) + \gamma_j^{(2)} (\bm{x}^{(2)}) > 0
\end{equation} 

Without loss of generality, we suppose $\gamma_j^{(1)} (\bm{x}) < \gamma_j^{(1)} (\bm{x}')$. Considering the conditions above, we have:
\begin{equation}
    \gamma_j^{(1)}  (\bm{x}^{(1)}) + \gamma_j^{(2)}  (\bm{x}'^{(2)}) < \gamma_j^{(1)}  (\bm{x}'^{(1)}) + \gamma_j^{(2)} (\bm{x}'^{(2)}) = 0
\end{equation}
Hence, there exist a $\hat{\bm{x}} = (\bm{x}^{(1)}, \hat{\bm{x}}^{(2)})$, subject to :
\begin{equation}
    \gamma_j^{(1)}  (\bm{x}^{(1)}) + \gamma_j^{(2)} (\hat{\bm{x}}^{(2)}) = 0 \quad \mathrm{and} \quad \|\bm{x} - \hat{\bm{x}}\|_2 < \|\bm{x} - \bm{x}'\|_2.
\end{equation}
Contradict with the assumption. 
Thus, the proposed bound can be verified.
\begin{equation}
\begin{aligned}
    P(\bm{x}) = \min_{\bm{x}'} \|\bm{x} - \bm{x}'\|_2 & \geq \frac{\sum_{m=1}^2 \left( a_y^{(m)}\tilde{W}^{(m)}_{y.} \phi^{(m)} (\bm{x}^{(m)})- a_j^{(m)}\tilde{W}^{(m)}_{j.} \phi^{(m)} (\bm{x}^{(m)})\right) + \tilde{\beta}_j  }{\sqrt{\sum_{m=1}^2 \left( a^{(m)}_y \tilde{\tau}^{(m)}_y + a^{(m)}_j \tilde{\tau}^{(m)}_j\right)^2  }}  \\
  &\quad \quad \quad \quad s.t.  ~~ \exists j \neq y, ~ \tilde{h}_y(\bm{x}') = \tilde{h}_j(\bm{x}').
\end{aligned}
\end{equation}
\end{proof}

\subsection{Smoothing process}
\label{Smoothing}
In this section, we provide the smoothing process of the margin regularization. Since the imbalance problem is introduced by modality preference, the unpreferred modality can barely obtain discriminative representation. Hence, we aim to enhance the unpreferred uni-modal representation learning, through enhancing its margin. In detail, we focus on the following objectives: 
\begin{equation}
\begin{aligned}
        \max_{\tilde{W}^{(m)}, \phi^{(m)}} \min_{m; k \neq y} \quad \tilde{W}^{(m)}_{y.} \phi^{(m)} (\bm{x}^{(m)})- \tilde{W}^{(m)}_{k.} \phi^{(m)} (\bm{x}^{(m)}).   
\end{aligned}
\end{equation}
This objective aims to improve the learning of the weaker uni-modality by encouraging the correct class and punishing others. For better optimization, we first the smooth minimization among class $j$ through LogSumExp function~\citep{nielsen2016guaranteed}:

\begin{equation}
    \begin{aligned}
                \max_{\tilde{W}^{(m)}, \phi^{(m)}}& \min_m \quad \tilde{W}^{(m)}_{y.} \phi^{(m)} (\bm{x}^{(m)}) - \log \left( \sum_{k \neq y} \exp(\tilde{W}^{(m)}_{k.} \phi^{(m)} (\bm{x}^{(m)}))\right),
    \end{aligned}
\end{equation}
and further, smooth the minimization among modality $m$, and transform it into minimization problem:
\begin{equation}
      \min_{\tilde{W}^{(m)}, \phi^{(m)}} ~~-\log \left(\sum_{m=1}^2 \frac{\exp(\tilde{W}^{(m)}_{y.} \phi^{(m)} (\bm{x}^{(m)}))}{\sum_{k \neq y} \exp( \tilde{W}^{(m)}_{k.} \phi^{(m)} (\bm{x}^{(m)})) }\right).
\end{equation}

Above all, we extend the regularization to the average of all samples and obtain the objective $L_1$:
\begin{equation}
    \begin{aligned}
    \min_{\tilde{W}^{(m)}, \phi^{(m)}}&~  \frac{1}{N}\sum_{i=1}^N \log \left(\sum_{m=1}^2 \frac{\sum_{k \neq y} \exp( \tilde{W}^{(m)}_{k.} \phi^{(m)} (\bm{x}_i^{(m)}))}{ \exp(\tilde{W}^{(m)}_{y.} \phi^{(m)} (\bm{x}_i^{(m)}))}\right),
    \end{aligned}
\end{equation}
which explicitly enlarges the uni-modal representation learning to solve the imbalance problem without changing the fusion factor $\bm{a}^{(m)}$, which can lead to a better uni-modal representation.

\subsection{Extensive analysis}

\paragraph{Extended to different fusion strategies.}
In the main manuscript, our analysis mainly focuses on the representative fusion strategy, late fusion, which is widely used in multi-modal research~\cite{huang2022modality, wang2020makes}. 
Meanwhile, our method are adaptable to other fusion mechanisms, where the modality-specific representations could interact earlier in the process. Previously, we defined the margin as $\zeta^{(m)}_j(\bm{x}^{(m)})$, where $j$ is the nearest class to calculate the margin, and the margin is only determined by uni-modal input $\bm{x}^{(m)}$. To adapt our method for these scenarios including intermediate fusion, we can redefine the representation margin as $\zeta^{(m)}_j(\bm{x}^{(1)}, \bm{x}^{(2)})$, rather than the previously, indicating that both modalities' input influence the margin. This modification allows us to extend the framework to measure multi-modal perturbations in a more integrated manner. Additionally, we can adapt the definition of the Lipschitz constant in our theoretical analysis here to measure how the multi-modal perturbation influences the margin:

\begin{equation}
\begin{aligned}
|\zeta_j^{(m)}(\bm{x}^{(1)}, \bm{x}^{(2)}) - \zeta_j^{(m)}(\bm{x}'^{(1)}, \bm{x}'^{(2)})| \le \tau_j^{(m1)}\left\| \bm{x}^{(1)} - \bm{x}'^{(1)} \right\|_2 + \tau_j^{(m2)}\left\| \bm{x}^{(2)} - \bm{x}'^{(2)} \right\|_2
\end{aligned}
\end{equation}
where $\tau_j^{(m1)}$ represents the Lipschitz constant of modality $m$ from modality $1$. This constant can reflect how the alteration in modality $1$ influences the margin in modality $m$. Then following the proof in the manuscript, we can observe that:
 \begin{equation}
\begin{aligned}
      c_j^{(1)} |\zeta_{j}^{(1)}(\bm{x}^{(1)}, \bm{x}^{(2)}) - \zeta_{j}^{(1)}(\bm{x}'^{(1)}, \bm{x}'^{(2)})|+ c_j^{(2)} |\zeta_{j}^{(2)}(\bm{x}^{(1)}, \bm{x}^{(2)}) - \zeta_{j}^{(2)}(\bm{x}'^{(1)}, \bm{x}'^{(2)})| \\         \leq  (c_j^{(1)} \tau_j^{(11)} + c_j^{(2)} \tau_j^{(21)})\left\|\bm{x}^{(1)} - \bm{x}'^{(1)}\right\|_2 + (c_j^{(1)} \tau_j^{(12)} + c_j^{(2)} \tau_j^{(22)})\left\|\bm{x}^{(2)} - \bm{x}'^{(2)}\right\|_2 .
\end{aligned}
\end{equation}
Thus, we can obtain the perturbation bound in this setting:
\begin{equation}
\begin{aligned}
 & \min_{\bm{x}'} \left\|\bm{x} - \bm{x}'\right\|_2 
 \geq \frac{c_j^{(1)} \zeta_{j}^{(1)}(\bm{x}^{(1)}, \bm{x}^{(2)}) + c_j^{(2)} \zeta_{j}^{(2)}(\bm{x}^{(1)}, \bm{x}^{(2)})+ \beta_j}{\sqrt{(c_j^{(1)} \tau_j^{(11)} + c_j^{(2)} \tau_j^{(21)})^2 +(c_j^{(1)} \tau_j^{(12)} + c_j^{(2)} \tau_j^{(22)})^2 }} \\
  whe&re   \quad j \neq y \quad s.t. \quad  c_j^{(1)} \zeta_{j}^{(1)}(\bm{x}'^{(1)}, \bm{x}'^{(2)}) + c_j^{(2)} \zeta_{j}^{(2)}(\bm{x}'^{(1)}, \bm{x}'^{(2)}) + \beta_j = 0.
\end{aligned}
\end{equation}

The idea of the proof is similar to the one in \autoref{Proof_theorem_1}. 

\paragraph{Extended to more modalities.}
 In this study, our analysis discusses the universal situation of two modalities, and can also be extended to the scenario with more than two modalities. To consider more modalities, the key is to introduce the margin of these modalities’s representation. Suppose we have $l$ different modality, the input of the $m$-th modality is $\bm{x}^{(m)}$, define the representation margin $\zeta_{j}^{(m)} (\bm{x}^{(m)})$, and the corresponding Lipschitz constant $\tau_j^{(m)}$. Thus, our bound can be extended to the following formulation.
 
\begin{equation}
\begin{aligned}
  \min_{\bm{x}'} \left\|\bm{x} - \bm{x}'\right\|_2 
 \geq \frac{ \sum_{m=1}^l c_j^{(m)} \zeta_{j}^{(m)}(\bm{x}^{(m)}) +  \beta_j}{\sqrt{ \sum_{m=1}^l (c_j^{(m)} \tau_j^{(m)})^2  }} \\
  where   \quad j \neq y \quad s.t. \quad   \sum_{m=1}^l c_j^{(m)} \zeta_{j}^{(m)}(\bm{x}'^{(m)}) +  \beta_j = 0.
\end{aligned}
\end{equation}

\section{Addition for Experiment}
\subsection{Details for datasets}
\label{detail_dataset}
We conduct experiments on three datasets Kinetics-Sounds  \citep{arandjelovic2017look}, UCF101~\citep{soomro2012ucf101}, and VGGSound~\citep{chen2020vggsound}. Here are the details.  

Kinetics-Sounds (KS) \citep{arandjelovic2017look} is a dataset containing 31 human action classes selected from Kinetics dataset~\citep{kay2017kinetics} which contains 400 classes of YouTube videos. All videos are manually annotated for human action using Mechanical Turk and cropped to 10 seconds long around the action. The 31 classes were chosen to be potentially manifested visually and aurally, such as playing various instruments. This dataset contains 22k 10-second video clips, and for each video, we select 3 frames (first, middle, and last frame)

UCF101 \citep{soomro2012ucf101} consists of in-the-wild videos from 101 action classes, it is typically regarded as a multi-modal dataset with RGB and optical flow modalities. This dataset contains 13,320 videos of human actions across 101 classes that have previously been used for video synthesis and prediction. For each video, we select the middle frame of RGB and 10 frames of optimal flow for training and prediction. And our backbone is the pre-trained ResNet 18. 

VGGSound(VGGS)~\citep{chen2020vggsound} is a large-scale video dataset that contains 309 classes, covering a wide range of audio events in everyday life. All videos in VGGSound are captured “in the wild” with audio-visual correspondence in the sense that the sound source is visually evident. The duration of each video is 10 seconds. In our experimental settings, 168,618 videos are used for training and validation, and 13,954 videos are used for testing because some videos are not available now on YouTube. For each video, we also select the middle frame for prediction.

\subsection{Effectiveness Validation}

\subsubsection{Validation across various models}

\paragraph{CRMT with different backbones} 
To evaluate the adaptability and effectiveness of our method across diverse frameworks, we conduct comprehensive experiments on the Kinetic-Sounds dataset utilizing different backbone architectures. These architectures encompass the use of ResNet34 for both Vision (V) and Audio (A) modalities, as well as an integrated approach combining ResNet18 for Vision (V) and a Transformer model for Audio (A). The results in \autoref{different_backbone} of these experiments demonstrate both the improvement and flexibility of our method when implemented across different backbones.

\begin{table}[htp!]
\caption{Comparative analysis of robustness across different backbones, comparing our method (CRMT-JT) with the joint training (JT) approach. }
\begin{tabular}{c|c|c|cc|cc}
\toprule
Backbone                                                                               & Method        & Clean  & Missing \#v
 & Missing \#a & FGM    & PGD-$\ell_2$ \\ \midrule
\multirow{2}{*}{\begin{tabular}[c]{@{}c@{}}ResNet34 (V) +\\ ResNet34 (A)\end{tabular}} & JT            & 0.6424 & 0.4528     & 0.2471     & 0.3132 & 0.2863 \\
                                                                                       & CRMT-JT & 0.7435 & 0.5269     & 0.5705     & 0.4978 & 0.4746 \\ \midrule
\multirow{2}{*}{\begin{tabular}[c]{@{}c@{}}Resnet18(V)\\ +Transformer(A)\end{tabular}} & JT            & 0.5538 & 0.3387     & 0.2703     & 0.2456 & 0.2100 \\
                                                                                       & CRMT-JT & 0.5807 & 0.3721     & 0.3539     & 0.3067 & 0.2711 \\ \bottomrule
\end{tabular}
\label{different_backbone}
\end{table}

\paragraph{CRMT with intermediate fusion strategies.}
In this manuscript, our analysis mainly focuses on the widely used late fusion method. 
To further validate our effectiveness, we delve into a widely used intermediate fusion strategy, the Multi-Modal Transfer Module (MMTM) \citep{vaezi20mmtm}, in which the interaction occurs among representations at mid-level. We compare MMTM and the one that implements our method (CRMT-MMTM) on the Kinetic-Sounds dataset. The experimental results in~\autoref{CRMT_MMTM} indicate that our method can also boost the performance and robustness for more multi-modal fusion strategies.

\begin{table}[htp!]
\centering
\caption{Experiment on the intermediate fusion strategy MMTM and with our method. }
\begin{tabular}{c|c|cc|cc}
\toprule
Method              & Clean  & Missing \#v & Missing \#a & FGM    & PGD-$\ell_2$ \\ \midrule
MMTM             & 0.6693 & 0.4542      & 0.2028      & 0.3438 & 0.3205       \\
CRMT-MMTM  & 0.6737 & 0.5211      & 0.3169      & 0.3445 & 0.3358       \\ \bottomrule
\end{tabular}
\label{CRMT_MMTM}
\end{table}

\paragraph{CRMT on pre-trained transformer.}
In the primary manuscript, we have verified the performance of our method using transformers trained from scratch. To extend the validation of our approach's effectiveness, we have conducted additional experiments utilizing an ImageNet-pretrained Transformer applied to the Kinetic-Sounds dataset. The findings of this extended analysis, detailed in \autoref{pretrained_trainsformer}, demonstrate that our method retains its effectiveness even when implemented with a pre-trained model. This further verifies the robustness and adaptability of our approach in diverse training scenarios. 

\begin{table}[htp!]
\centering
\caption{Experiment on Multi-modal Transformer with pre-trained on KS dataset. }
\begin{tabular}{c|c|cc}
\toprule
Transformer    & Clean  & FGM    & PGD-$\ell_2$ \\ \midrule
Baseline             & 0.6788 & 0.1366 & 0.0865       \\
CRMT (ours)  & 0.7406 & 0.3198 & 0.2078       \\ \bottomrule
\end{tabular}
\label{pretrained_trainsformer}
\end{table}

\subsubsection{Validation under different setting}

\paragraph{Robustness evaluation under various multi-modal attack. }
In the Experiment section, we initially describe two multi-modal attack methods: Fast Gradient Method (FGM) and Projected Gradient Descent with $\ell_2$ norm (PGD-$\ell_2$). To further substantiate the robustness of our approach, this section introduces and evaluates three additional multi-modal attack methods. Firstly, we implement the Multi-modal Embedding Attack (MEA) method as proposed by \citet{zhang2022towards}, which focuses on shifting the joint representation rather than altering the prediction output. Secondly, we explore the efficacy of our method against the Multi-modal Gaussian Noise (MGN) and the Multi-modal Pixel Missing (MPM) attacks. The results in \autoref{mm_attack} of these expanded experiments demonstrate that our methodology is not only robust to the previously examined attacks (FGM and PGD-$\ell_2$) but also capable of defending various multi-modal attacks.

\begin{table}[htp!]
\centering
\caption{Comparison of robustness against different multi-modal attacks on KS dataset. }
{\footnotesize
\setlength{\tabcolsep}{5pt}
\begin{tabular}{c|c|ccc|ccc|ccc}
\toprule
    & JT    & GB    & OGM   & PMR   & AT    & Mixup & MSEFM & CMRT-JT & CRMT-Mix & CRMT-AT \\ \midrule
MGN & 0.525 & 0.462 & 0.527 & 0.480 & 0.524 & 0.501 & 0.477 & 0.560   & 0.634    & 0.604   \\
MPM & 0.315 & 0.344 & 0.446 & 0.328 & 0.370 & 0.467 & 0.298 & 0.507   & 0.570    & 0.548   \\
MEA & 0.340 & 0.412 & 0.363 & 0.455 & 0.588 & 0.478 & 0.500 & 0.556   & 0.565    & 0.689   \\  \bottomrule
\end{tabular}
}
\label{mm_attack}
\end{table}

\paragraph{Extended to the UCF dataset with three modalities.}
During our analysis and experiments, we consider the universal situation containing two modalities. Our method can also be extended to the scenario with more than two modalities. We conduct experiments on the UCF101 dataset using three modalities: RGB, Optical Flow, and RGB Frame Difference. 
We compared the performance of methods for both trained from scratch and with pre-training ImageNet-pretrained ResNet18. The experimental results detailed in \autoref{three_modality}. The outcomes demonstrate the effectiveness of our method in enhancing multi-modal robustness in more than two modalities.

\begin{table}[htp!]
\centering
\caption{Comparison on UCF101 with three modalities. }
\begin{tabular}{c|cc|cc}
\toprule
             & \multicolumn{2}{c|}{Training from scratch} & \multicolumn{2}{c}{Pretrained} \\ \midrule
Attack       & JT                  & CRMT-JT              & JT            & CRMT-JT        \\ \midrule
Clean        & 0.4490              & 0.5640               & 0.8312        & 0.8506         \\
FGM          & 0.4005              & 0.4567               & 0.2471        & 0.4138         \\
PGD-$\ell_2$ & 0.3963              & 0.4312               & 0.0783        & 0.2623         \\ \bottomrule
\end{tabular}
\label{three_modality}
\end{table}

\paragraph{Experiment on Image-Text dataset.}
We apply experiments on vision-text classification tasks to verify our effectiveness. We utilize a widely used Food101 dataset, which inputs an image-text pair for classification. We employ a Vision Transformer (ViT) as our image encoder and BERT as our text encoder, subsequently concatenating their outputs to achieve a unified joint representation. To evaluate robustness, we apply multiple attacks including modality missing and descent-based attacks (FGM and PGD-$\ell_2$). It is important to note that attacks like the Fast Gradient Method (FGM) and Projected Gradient Descent with $\ell_2$ norm (PGD-$\ell_2$) are typically applied to continuous data. Given that text represents discontinuous data, we focus on implementing these attack methods on the image modality. Our results reveal that the text modality is more critical than the image modality, as its absence significantly impacts model performance, as shown in \autoref{image_text}. Concentrating on the $\ell_2$-norm, we achieve enhanced robustness under both FGM and PGD-$\ell_2$ attack. Our method demonstrates a notable performance increase in scenarios where text is absent, though there is a slight decline in performance when the image modality is missing. This is attributed to the huge performance difference between text and image.
\begin{table}[htp!]
\centering
\caption{Comparison of our proposed CRMT method on Image-Text dataset Food101. }
\begin{tabular}{c|c|cc|cc}
\toprule
Methods & Clean  & FGM on Image & PGD-$\ell_2$ on Image & Missing Text & Missing Image \\ \midrule
JT      & 0.8218 & 0.7603       & 0.7281                & 0.0497    & 0.7831    \\
CRMT-JT & 0.8257 & 0.7656       & 0.7313                & 0.0759    & 0.7779    \\ \bottomrule
\end{tabular}
\label{image_text}
\end{table}

\paragraph{Comparison with additional robust methods}
In addition to the comparative analyses presented in the Experiment section, we further extend the validation of our method by contrasting it with additional multi-modal robustness approaches. Specifically, we compare our approach with two recent methods: Robust Multi-Task Learning (RMTL) as described by \citet{ma2022multimodal} and Uni-Modal Ensemble with Missing Modality Adaptation (UME-MMA) proposed by \citet{li2023makes}. For the RMTL implementation, we adopt a multi-task learning framework to develop a robust multi-modal model, incorporating both full-modal and modality-specific tasks. These methods are applied and compared to the Kinetic-Sounds dataset. As illustrated in \autoref{more_robustness}, our method demonstrates superior performance in scenarios involving both modality absence and targeted attacks. These results present the robustness and effectiveness of our approach and the validity of our proposed analysis.

\begin{table}[htp!]
\centering
\caption{Comparison with recent multi-modal attack approaches. }
\begin{tabular}{c|c|cc|cc}
\toprule
Method  & Clean  & Missing \#v & Missing \#a & FGM    & PGD-$\ell_2$ \\ \midrule
UME-MMA & 0.6999 & 0.5334      & 0.4666      & 0.3394 & 0.3125       \\
RMTL    & 0.6672 & 0.5015      & 0.2994      & 0.3641 & 0.3445       \\
CRMT-JT & 0.7580 & 0.5596      & 0.5908      & 0.4906 & 0.4680       \\ \bottomrule
\end{tabular}
\label{more_robustness}
\end{table}

\begin{table}[htp!]
\caption{Performance against different sizes of FGM attack on UCF101 dataset to evaluate the robustness of different iterations in CRMT.}
\vspace{-1em}
\label{different_iter}
{\footnotesize
\begin{tabular}{c|ccccccccc}
\toprule
Attack Size   & 0.0    & 0.2    & 0.4    & 0.7    & 1.0    & 1.5    & 2.0    & 2.5    & 3.0    \\ \midrule
Iter 1 & 0.7888 & 0.5511 & 0.4605 & 0.4120 & 0.3851 & 0.3841 & 0.3732 & 0.3647 & 0.3562 \\ 
Iter 2 & 0.7888 & 0.5514 & 0.4589 & 0.4029 & 0.3864 & 0.3851 & 0.3737 & 0.3648 & 0.3559 \\
Iter 4 & 0.7891 & 0.5543 & 0.4660 & 0.4113 & 0.3921 & 0.3881 & 0.3769 & 0.3705 & 0.3629 \\ \bottomrule
\end{tabular}
}
\end{table}
\subsection{Iterations in training procedure}

\label{Iterations}

In the method section's concluding part, we currently perform two sequential steps at once. However, it is worth noting that in various applications, incorporating more iterations in the training procedure is a common practice, which holds the potential to further enhance the robustness of our method.
To validate the validity of our approach, we conducted experiments by extending our method to incorporate more iterations. As depicted in \autoref{different_iter}, the results of employing additional iterations demonstrate improved robustness capabilities. But only one iteration of our method can achieve considerable robustness, without a huge consumption of training cost.
\subsection{Experiments about modality preference}

\subsubsection{Verification of vulnerable modality}
As shown in \autoref{preference} in the manuscript, we demonstrate how the ratio of the vulnerability indicator ($\eta$) varies in our method. Furthermore, to clearly explain this phenomenon, in \autoref{etas} (a, b), we provide the heat map of the indicator $\eta$ to represent the robustness of each uni-modality, where the smaller indicator means the modality is more robust. Meanwhile, we also provide the uni-modal perturbation radius to further verify this modality preference, where we list the percentage of safe samples that can always resist this size of perturbation, as shown in \autoref{etas} (c). 
It can be seen that when joint training is applied to the Kinetic-Sounds dataset, the audio modality is definitely more vulnerable than the vision modality, thus explaining the phenomenon in \autoref{teasor} in the paper. 

\begin{figure}[htp!]
\vspace{-3mm}
\centering
\subcaptionbox{$\eta^{(v)}$ in JT. }{
\includegraphics[width=0.31\textwidth]{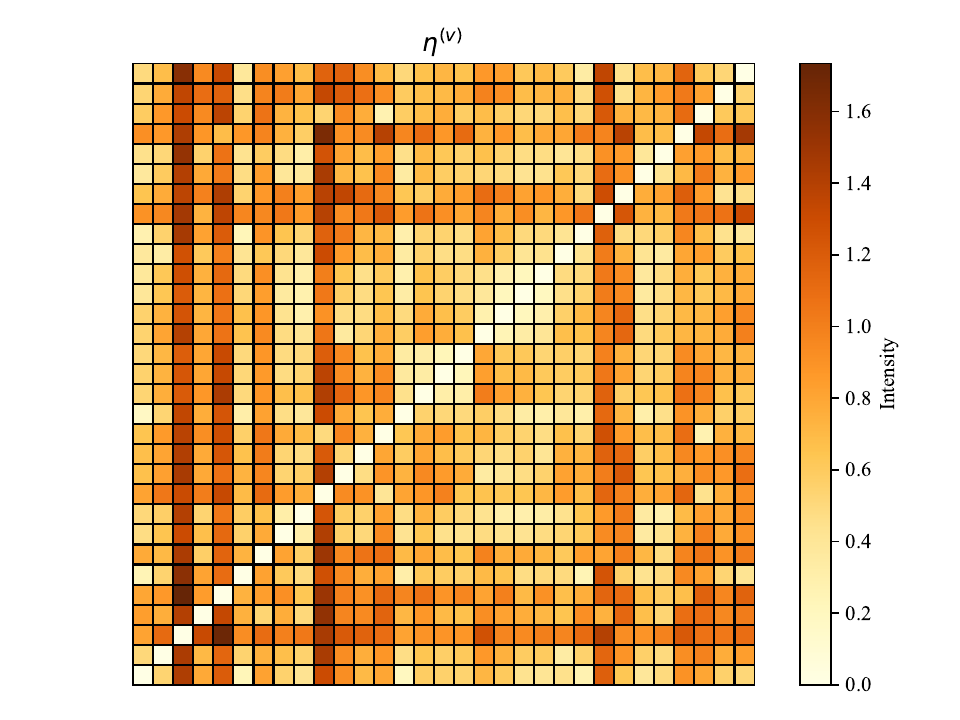}}
\subcaptionbox{$\eta^{(a)}$ in JT.}{
\includegraphics[width=0.31\textwidth]{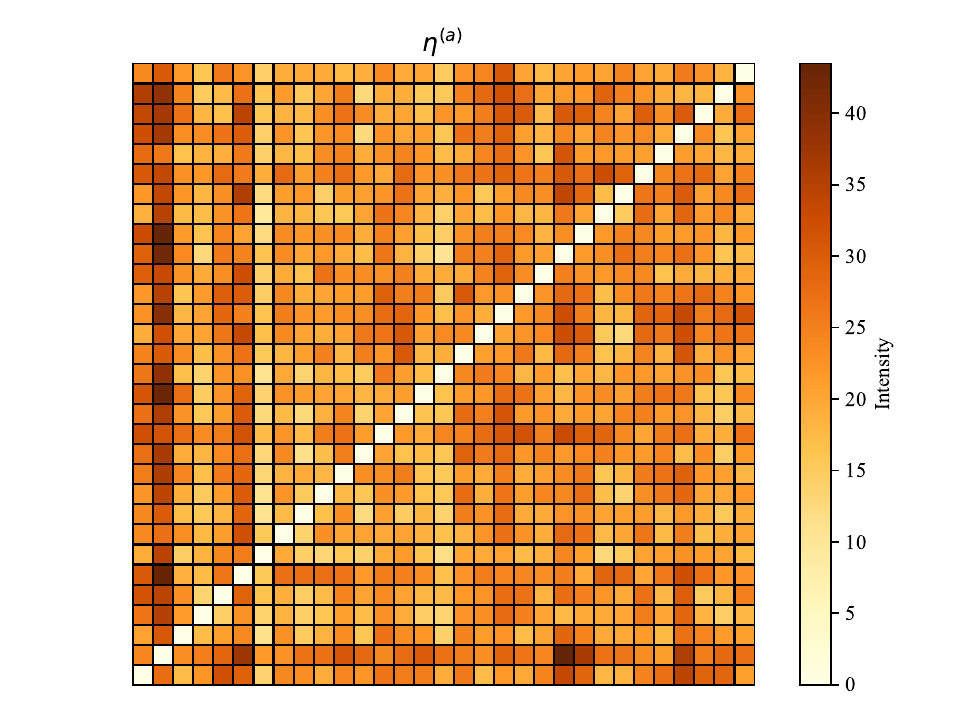}}
\subcaptionbox{Uni-modal perturbation bound.}{
\includegraphics[width=0.31\textwidth]{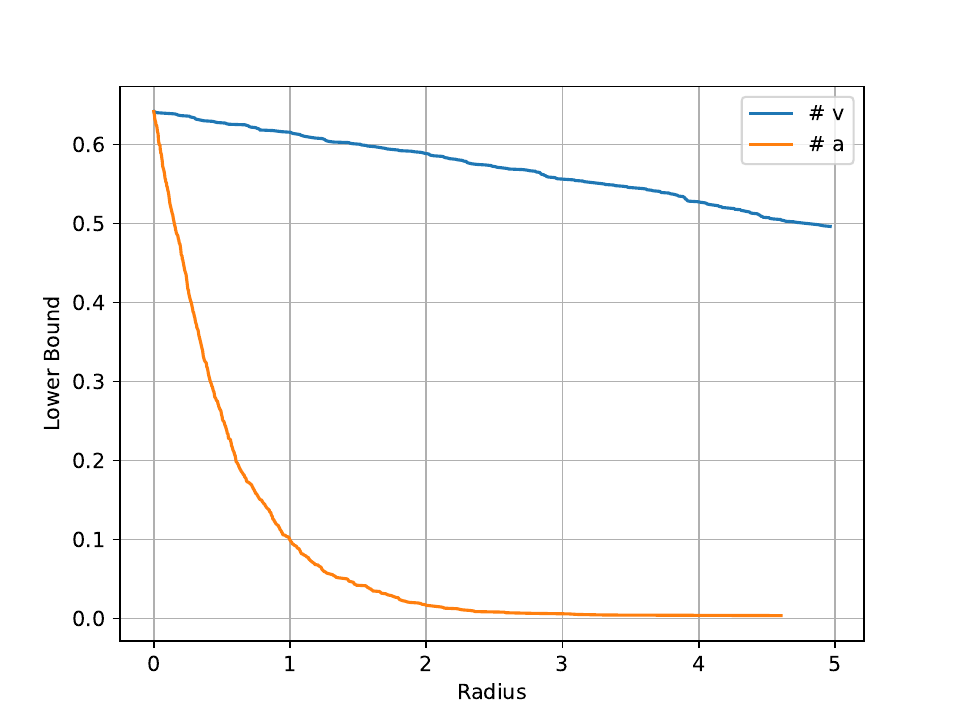}}
\vspace{-2mm}
\caption{Evaluation of the vulnerability indicators $\eta$ between modality \#$v$ and \#$a$ (preferred) and the comparison of safe samples against uni-modal perturbation in the Joint Training (JT) method. 
}
\label{etas}
\vspace{-1mm}
\end{figure}

\subsubsection{Uni-modal representation analysis}

In this part, we verify the imbalanced problem brought by modality preference, which results in the low quality of uni-modal representation. Here, we apply \textit{t-distributed Stochastic Neighbor Embedding} (t-SNE) method on each uni-modal representation and illustrate the visualization of uni-modal representation on KS dataset. As shown in \autoref{tsne}, in JT, MMAT, and Mixup, the quality of representation in modality \#$a$ is better than modality \#$v$, which implies the margins of representation \#$v$ are relatively small. Moreover, by applying our training procedure (see CRMT-JT, CRMT-AT, CRMT-Mix), representations of both modalities show good discrimination, indicating that we enlarge the uni-modal representation margins and eventually enhance the robustness. 

\begin{figure}[htp!]
\centering
\subcaptionbox{JT modality \#$v$. }{
\includegraphics[width=0.235\textwidth]{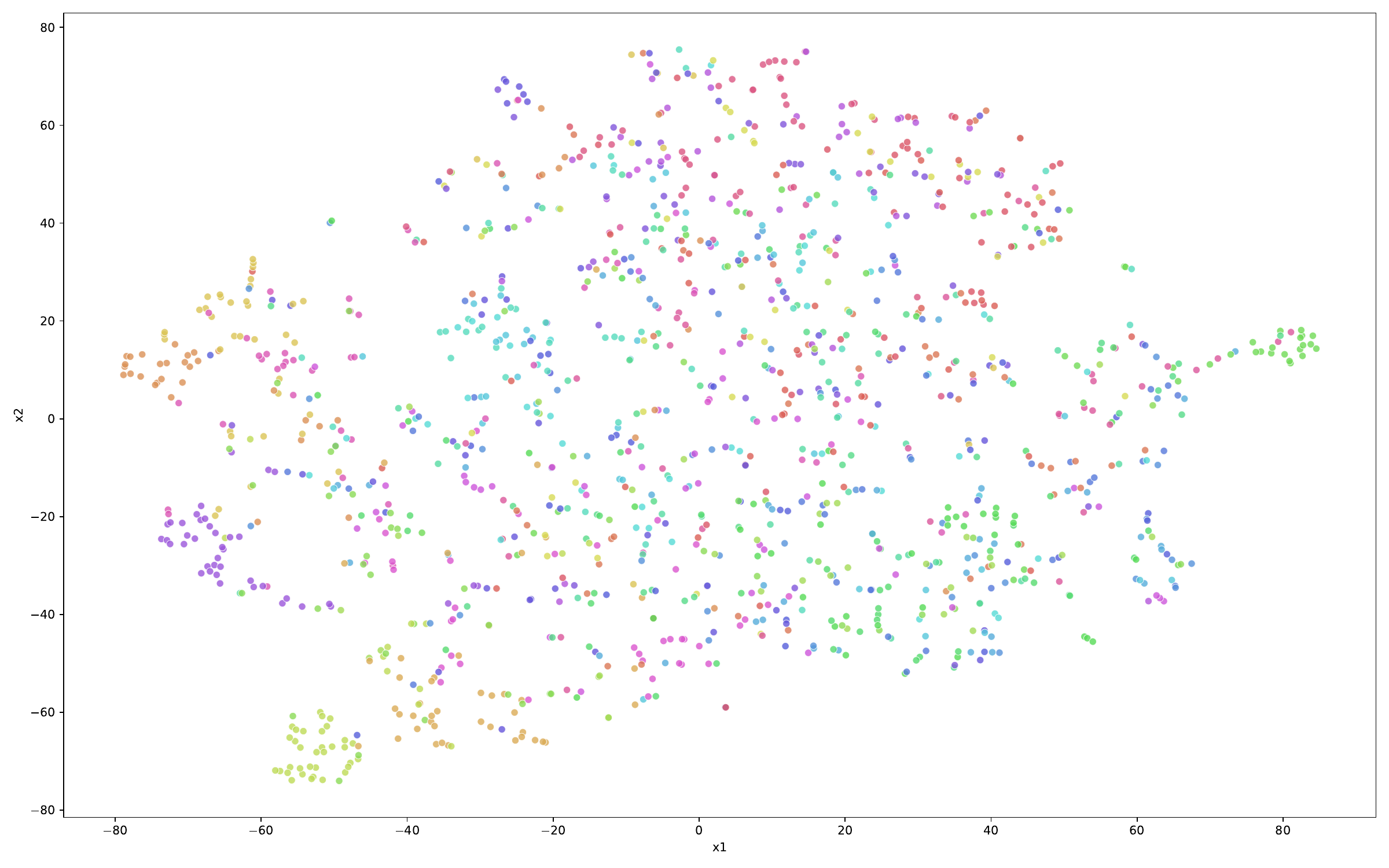}}
\subcaptionbox{JT modality \#$a$.}{
\includegraphics[width=0.235\textwidth]{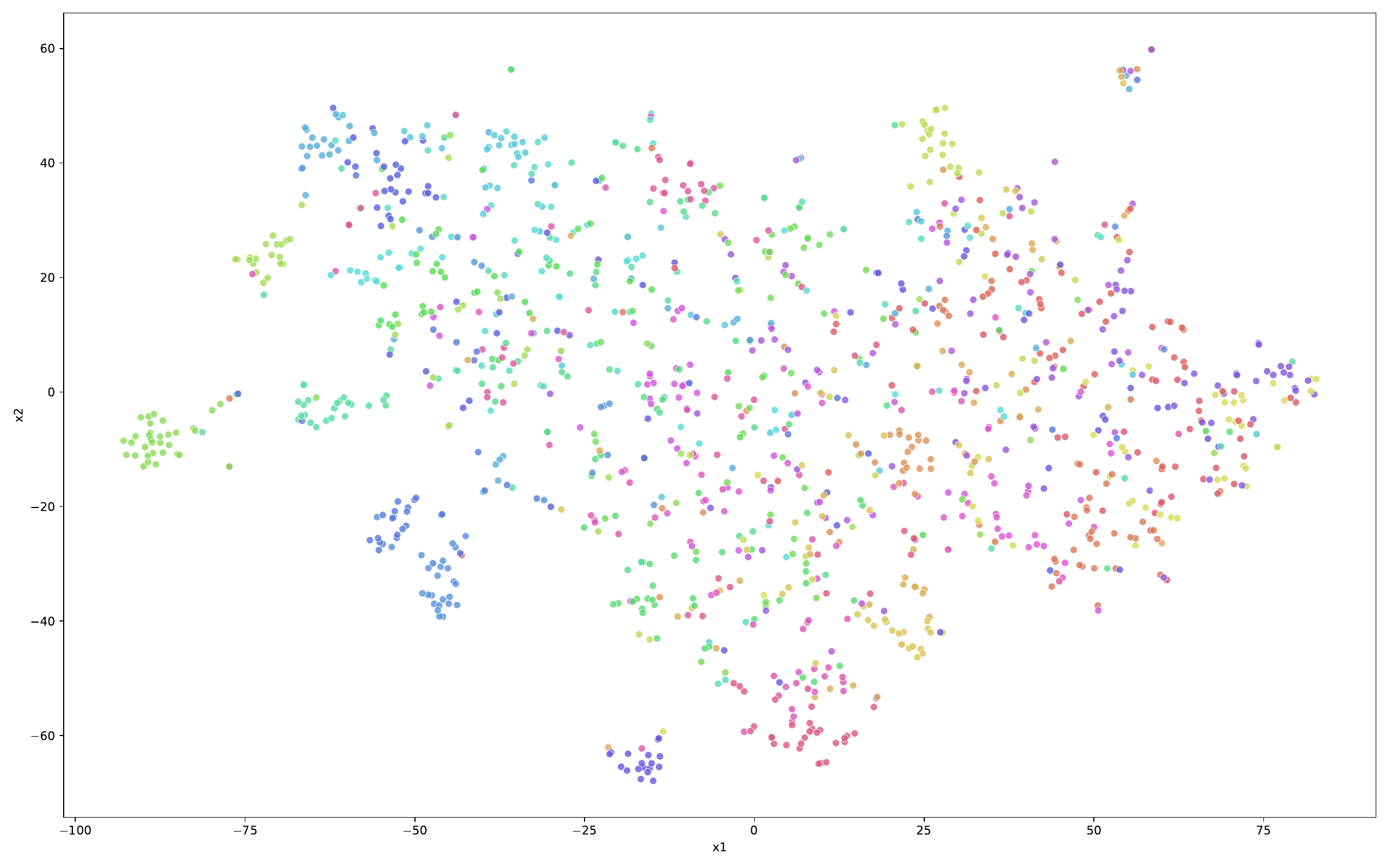}}
\subcaptionbox{CRMT-JT modality \#$v$. }{
\includegraphics[width=0.235\textwidth]{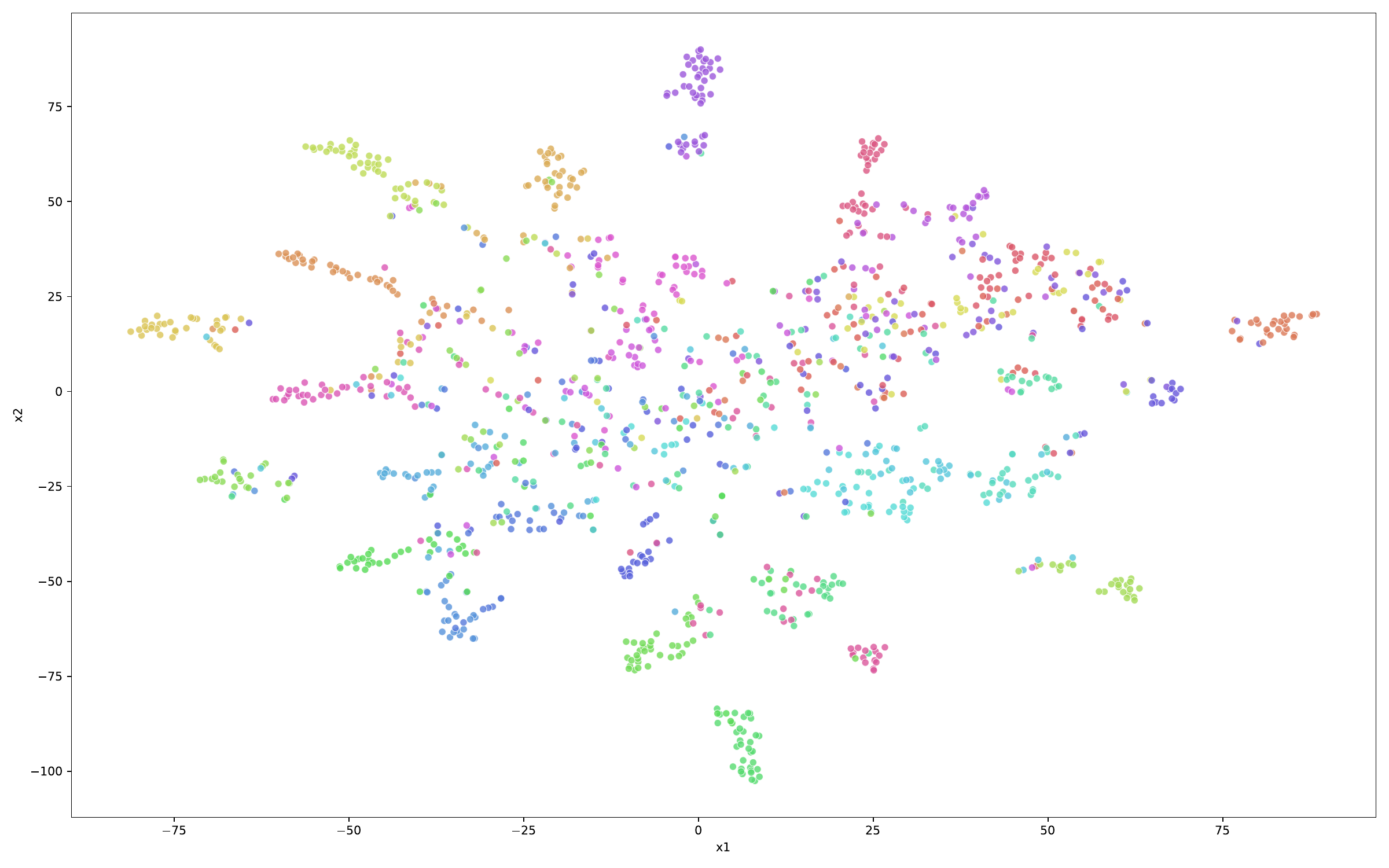}}
\subcaptionbox{CRMT-JT modality \#$a$.}{
\includegraphics[width=0.235\textwidth]{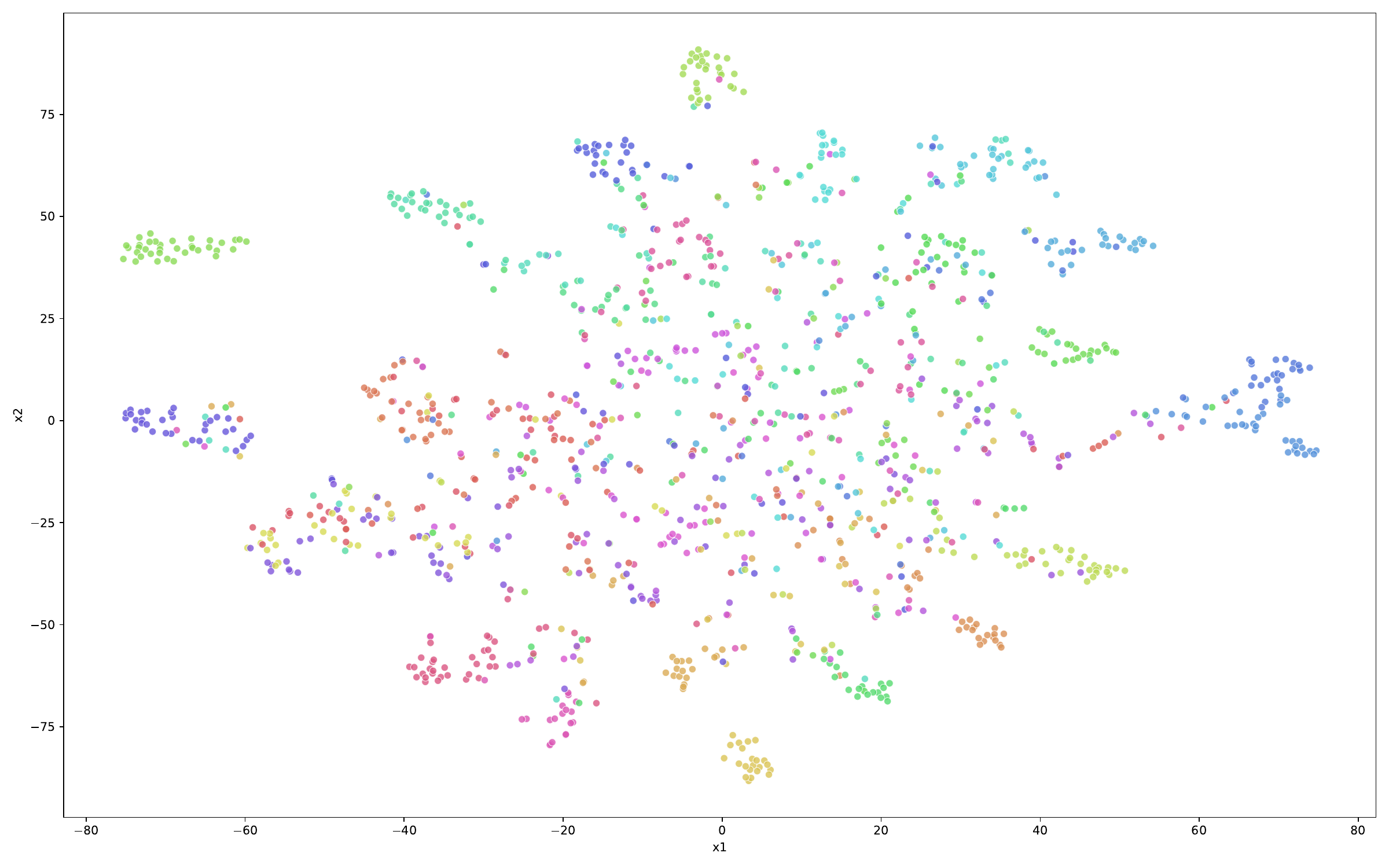}}
\subcaptionbox{MMAT modality \#$v$. }{
\includegraphics[width=0.235\textwidth]{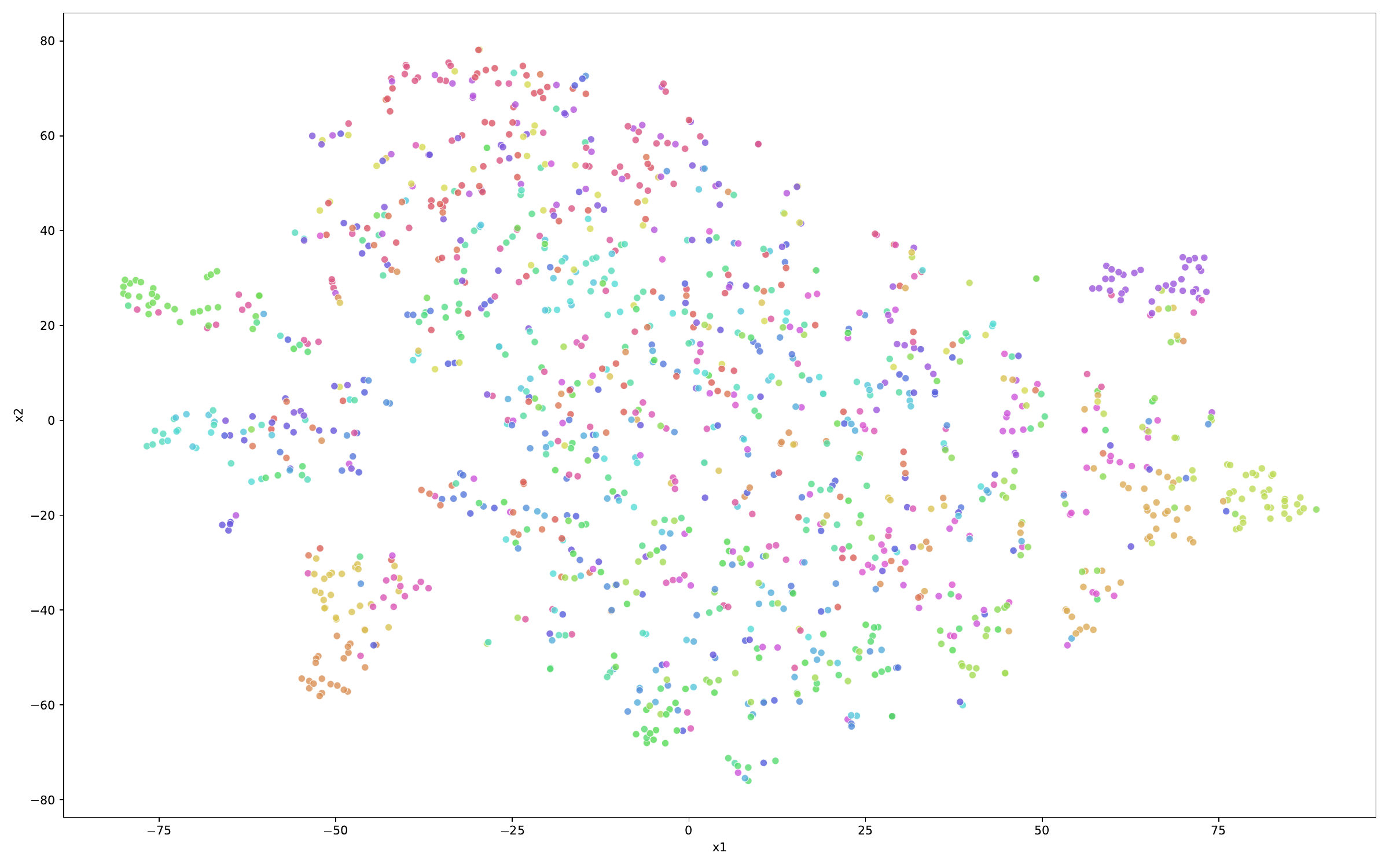}}
\subcaptionbox{MMAT modality \#$a$.}{
\includegraphics[width=0.235\textwidth]{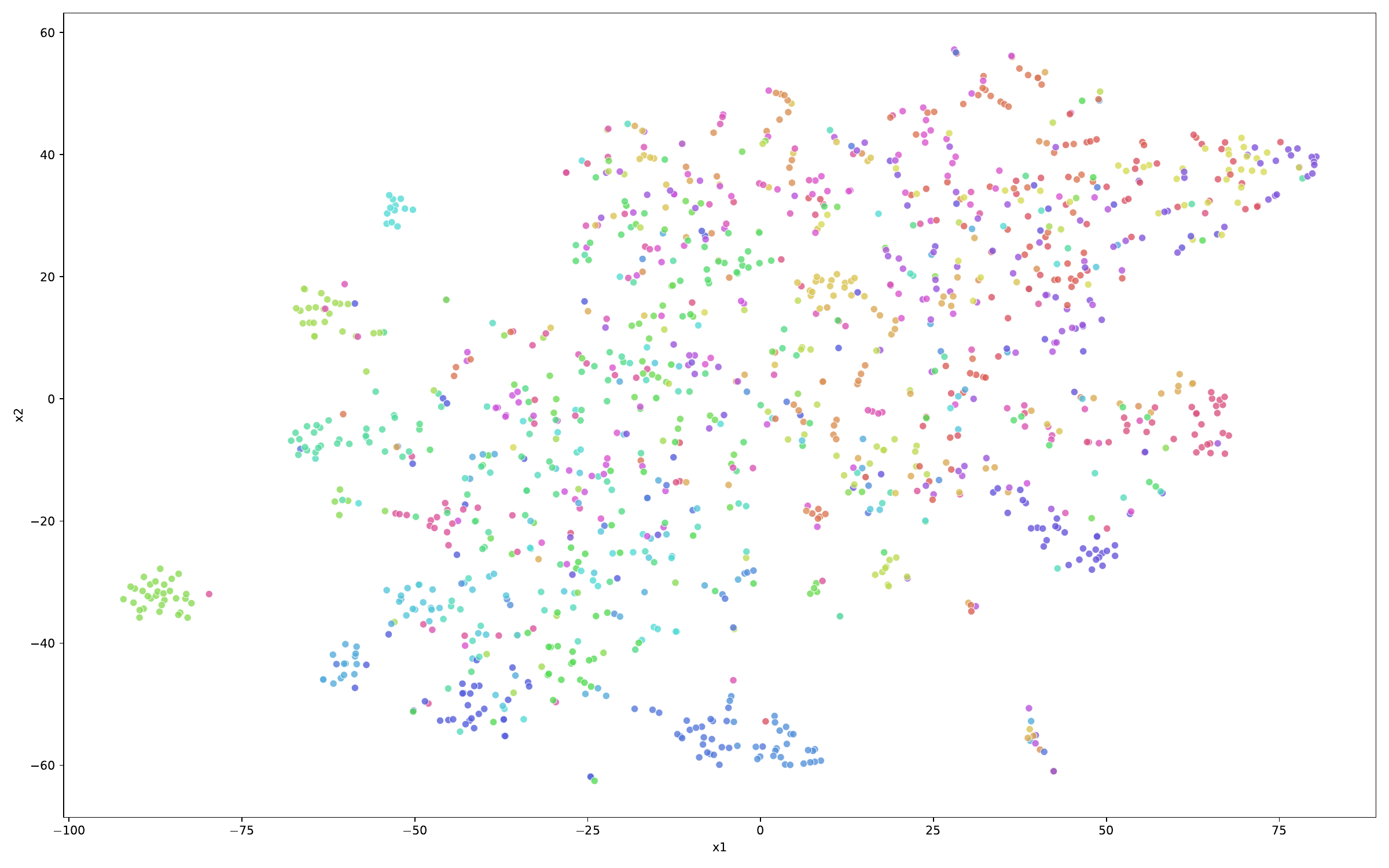}}
\subcaptionbox{CRMT-AT modality~\#$v$. }{
\includegraphics[width=0.235\textwidth]{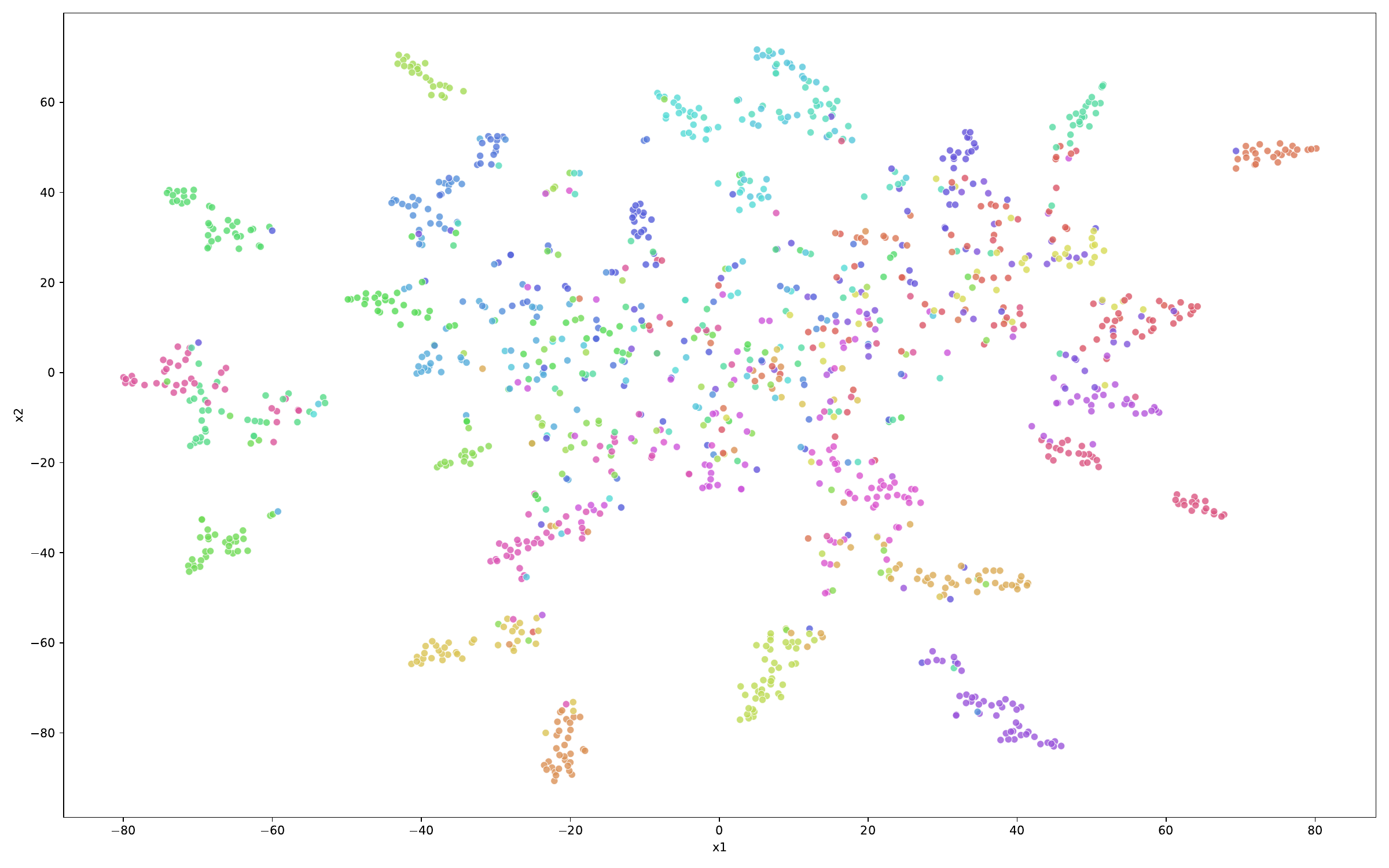}}
\subcaptionbox{CRMT-AT modality~\#$a$.}{
\includegraphics[width=0.235\textwidth]{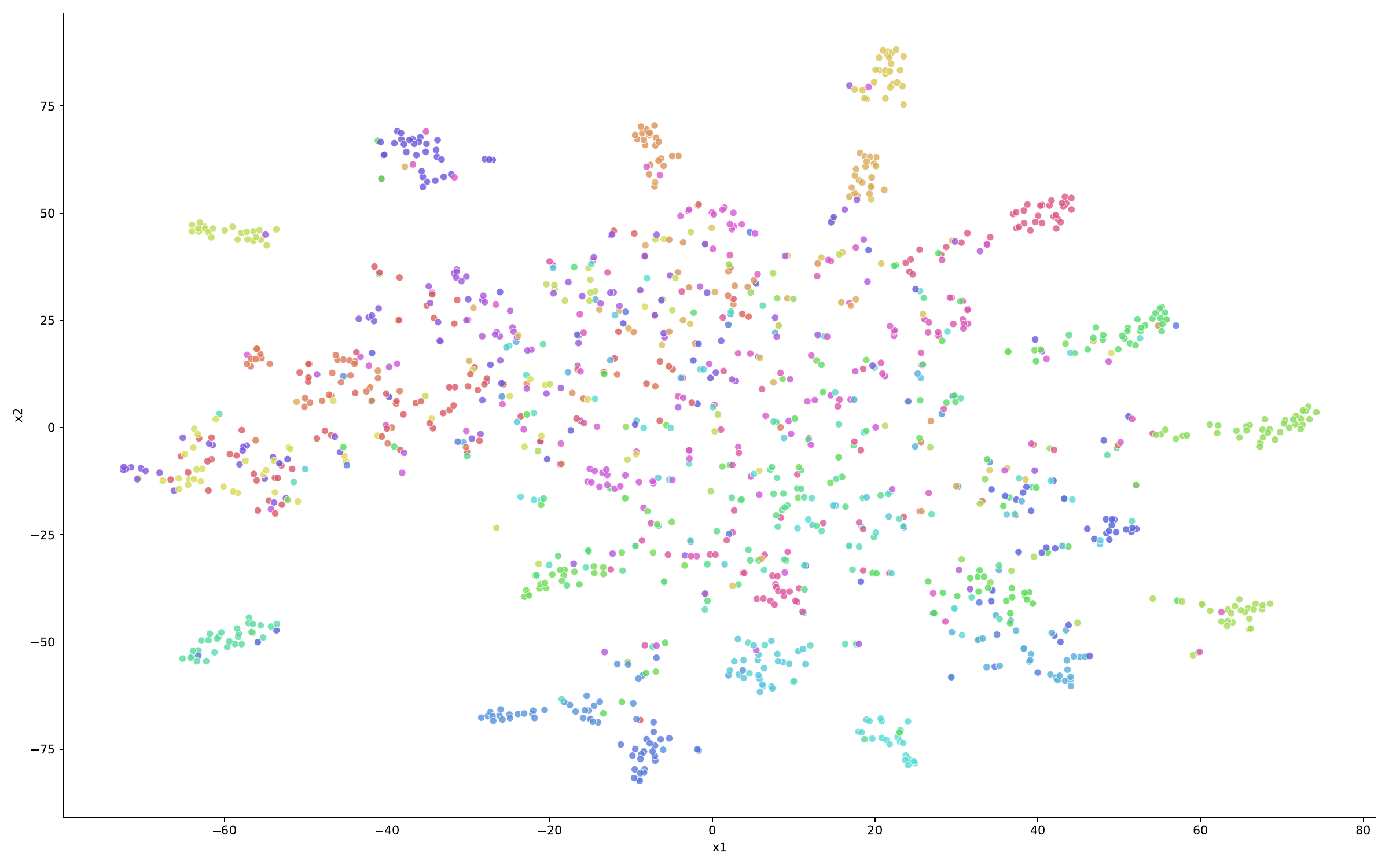}}

\subcaptionbox{Mixup modality \#$v$. }{
\includegraphics[width=0.238\textwidth]{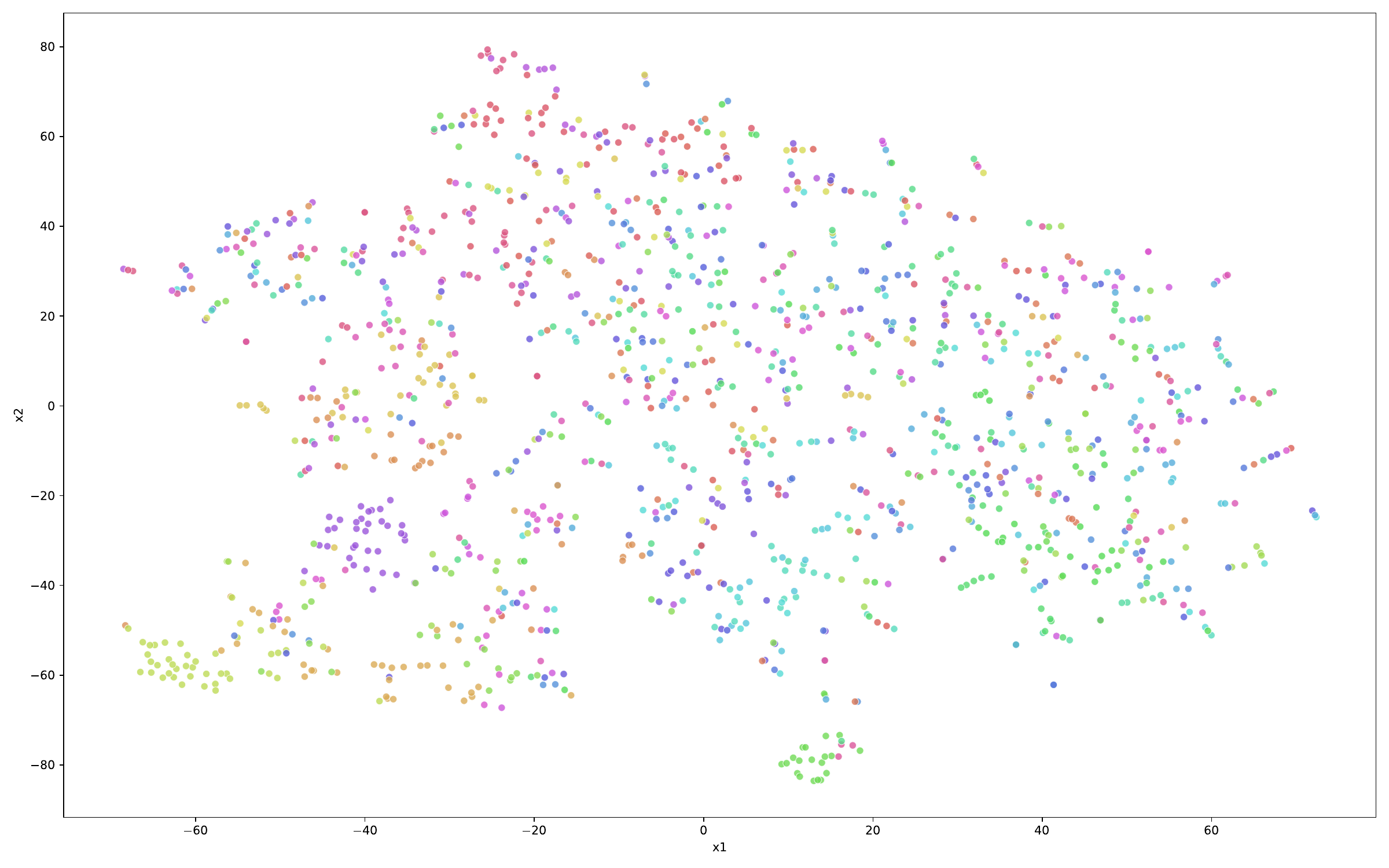}}
\subcaptionbox{Mixup modality \#$a$.}{
\includegraphics[width=0.238\textwidth]{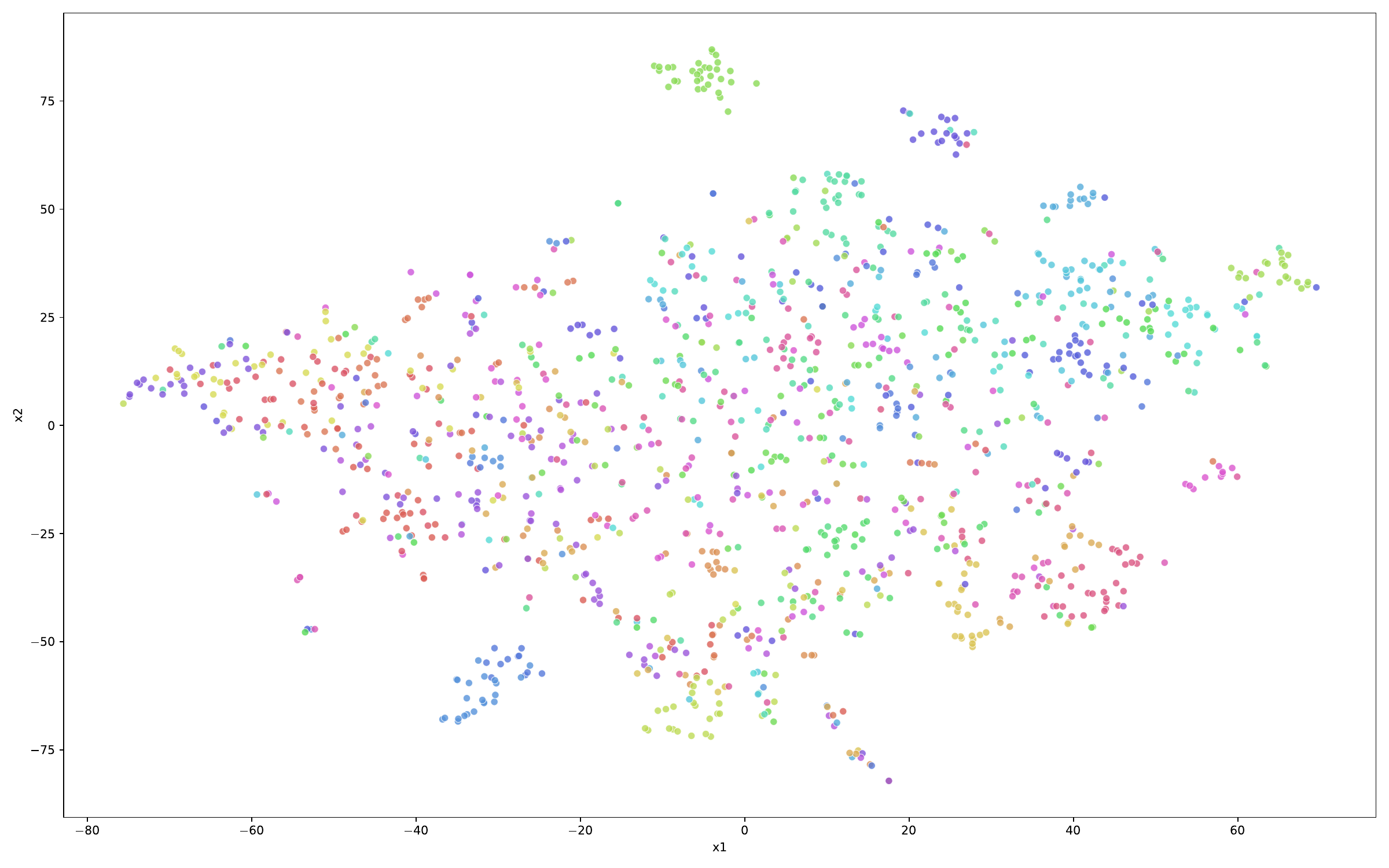}}
\subcaptionbox{CRMT-Mix modality \#$v$. }{
\includegraphics[width=0.238\textwidth]{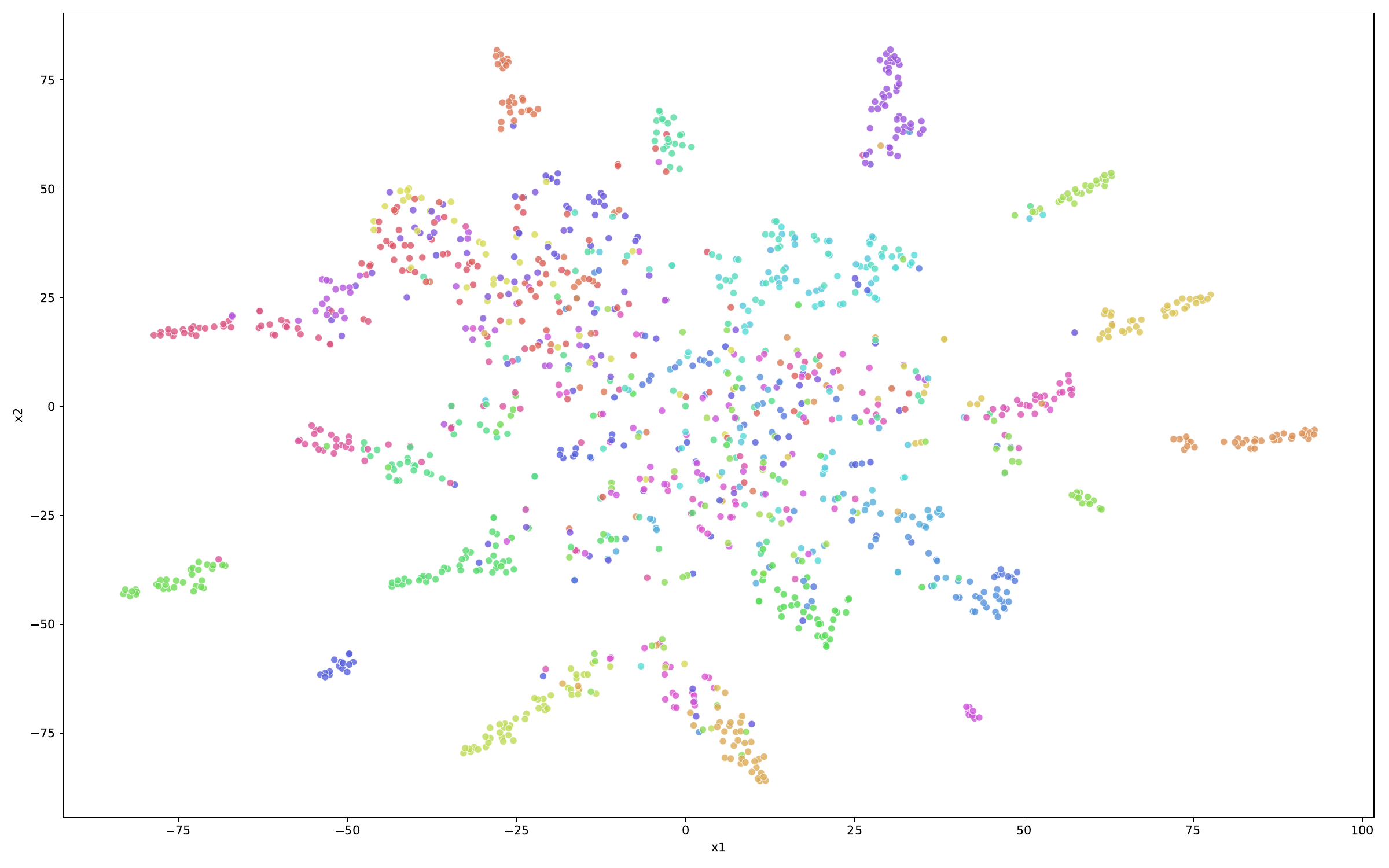}}
\subcaptionbox{CRMT-Mix modality \#$a$.}{
\includegraphics[width=0.238\textwidth]{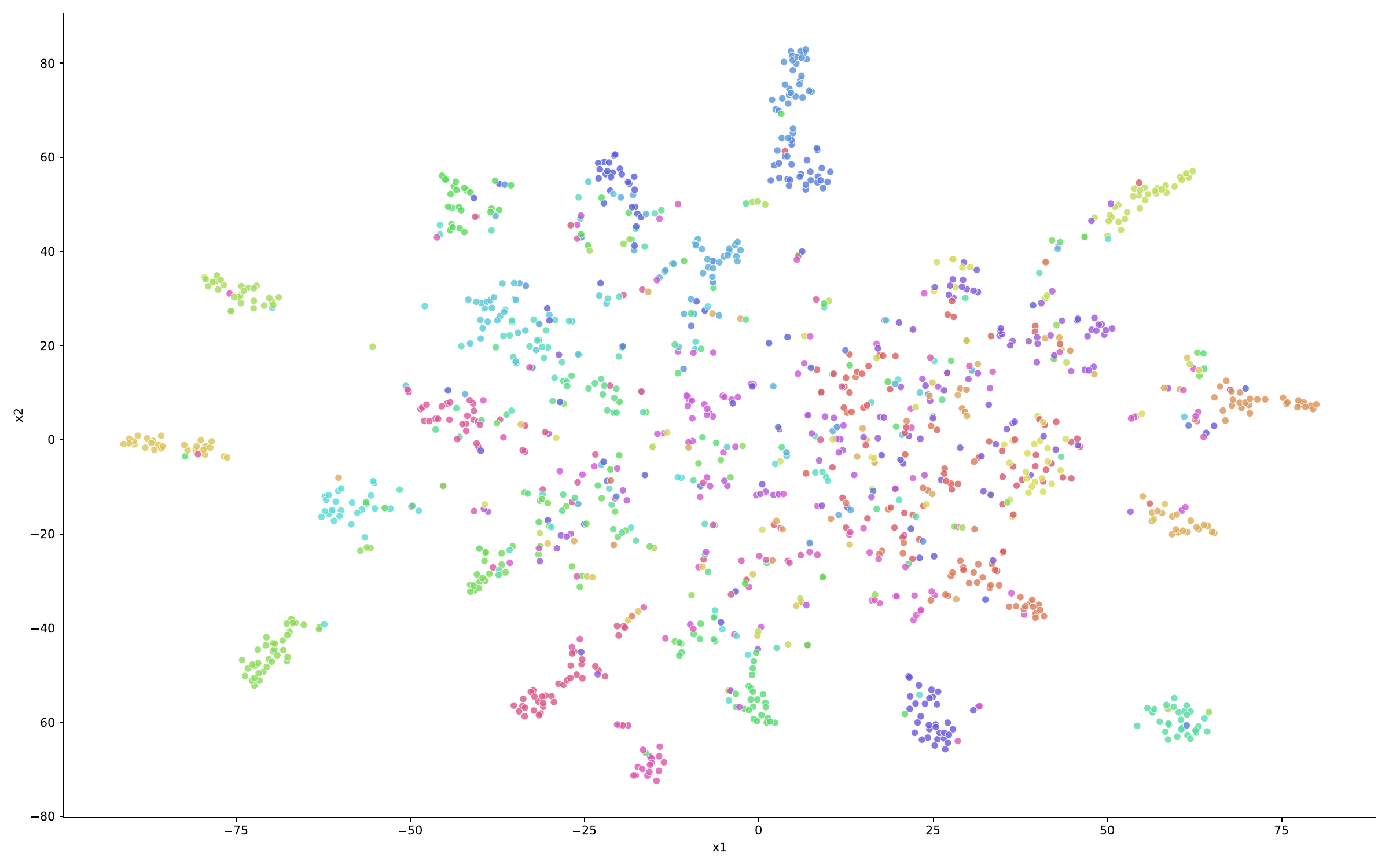}}
\caption{An experiment analyze the uni-modal representation of modality \#$v$ and \#$a$ (preferred) using t-SNE. We compare different methods and our corresponding CRMT procedure (including CRMT-JT, CRMT-AT, CRMT-Mix) and illustrate that our methods well learn the uni-modal representation. }
\label{tsne}
\end{figure}

\begin{figure}[htp!]
\vspace{-2em}
\centering
\subcaptionbox{$\eta^{(v)}/\eta^{(a)}$ in JT. }{
\includegraphics[width=0.32\textwidth]{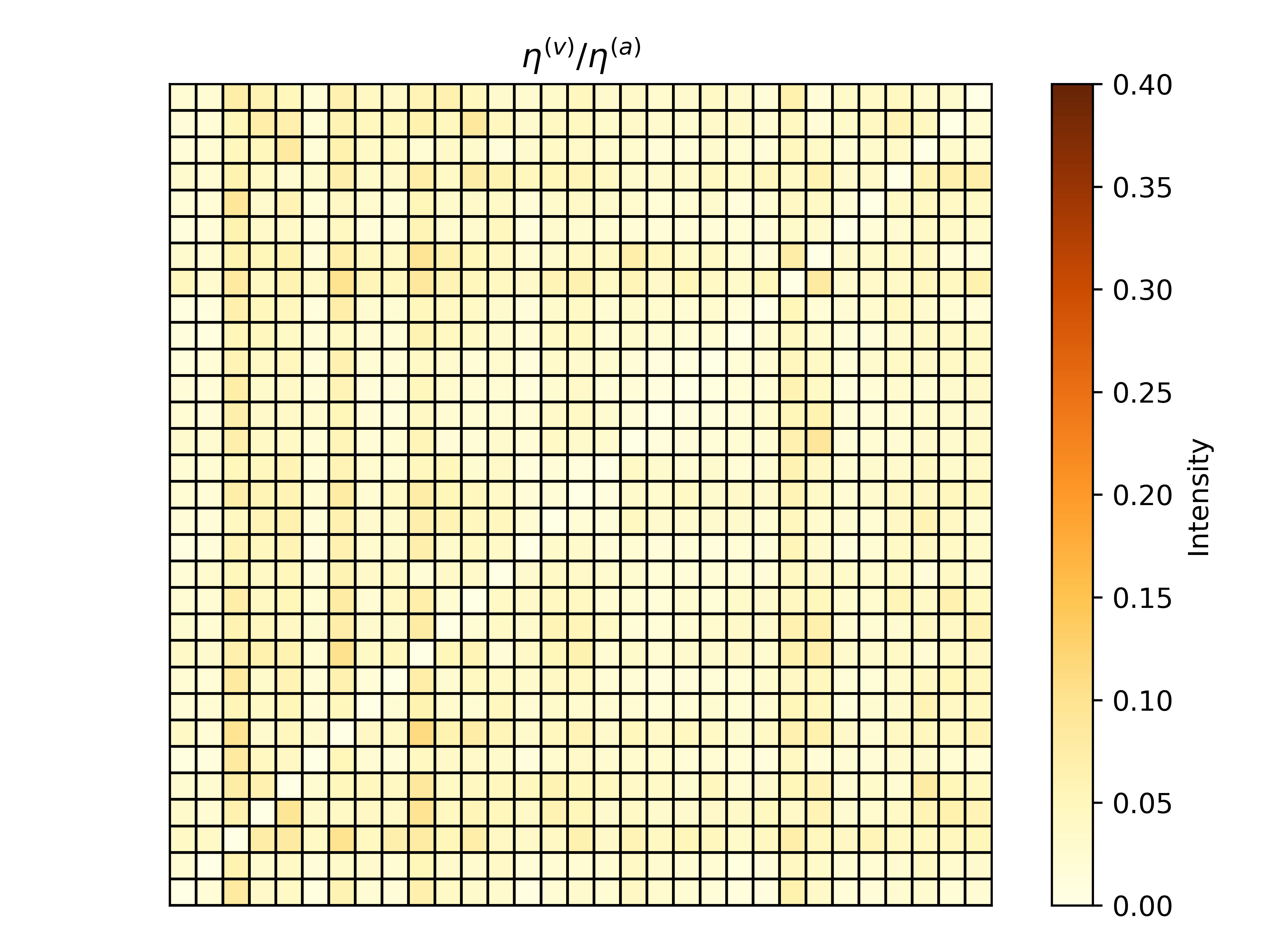}}
\subcaptionbox{$\eta^{(v)}/\eta^{(a)}$ in OJT.}{
\includegraphics[width=0.32\textwidth]{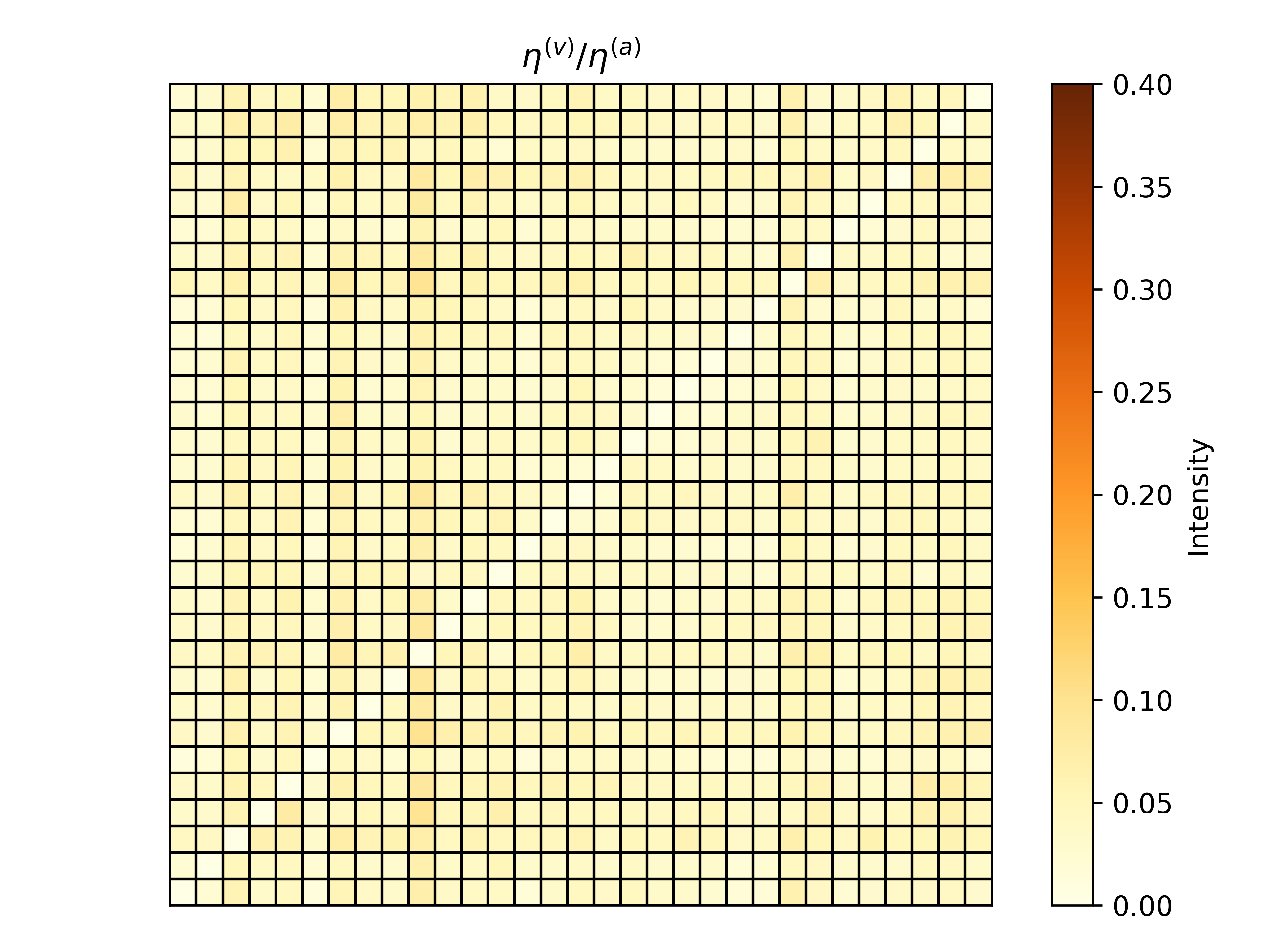}}
\subcaptionbox{$\eta^{(v)}/\eta^{(a)}$ in CRMT step-1. }{
\includegraphics[width=0.32\textwidth]{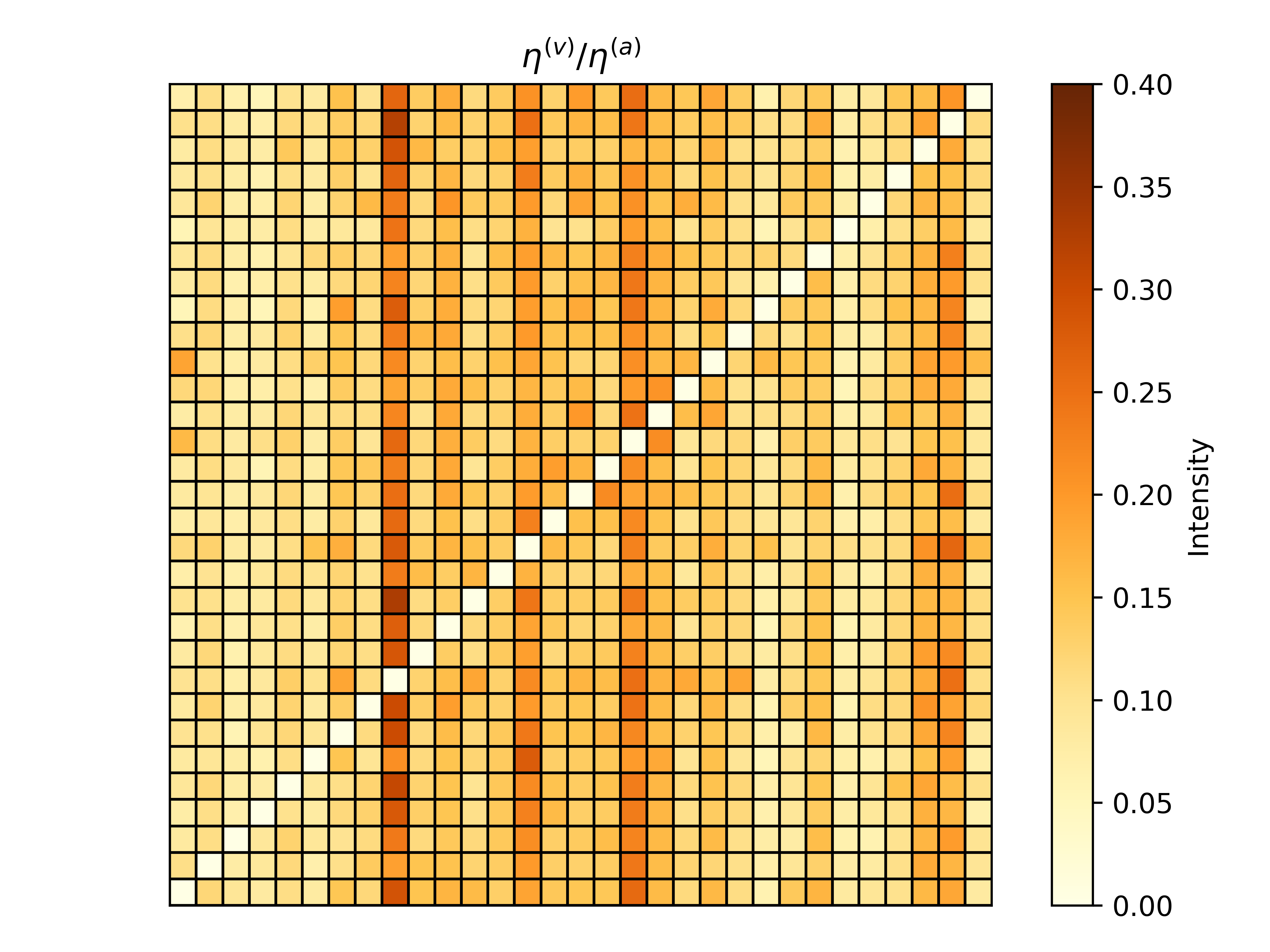}}
\subcaptionbox{$\eta^{(v)}/\eta^{(a)}$ in CRMT step-2.}{
\includegraphics[width=0.32\textwidth]{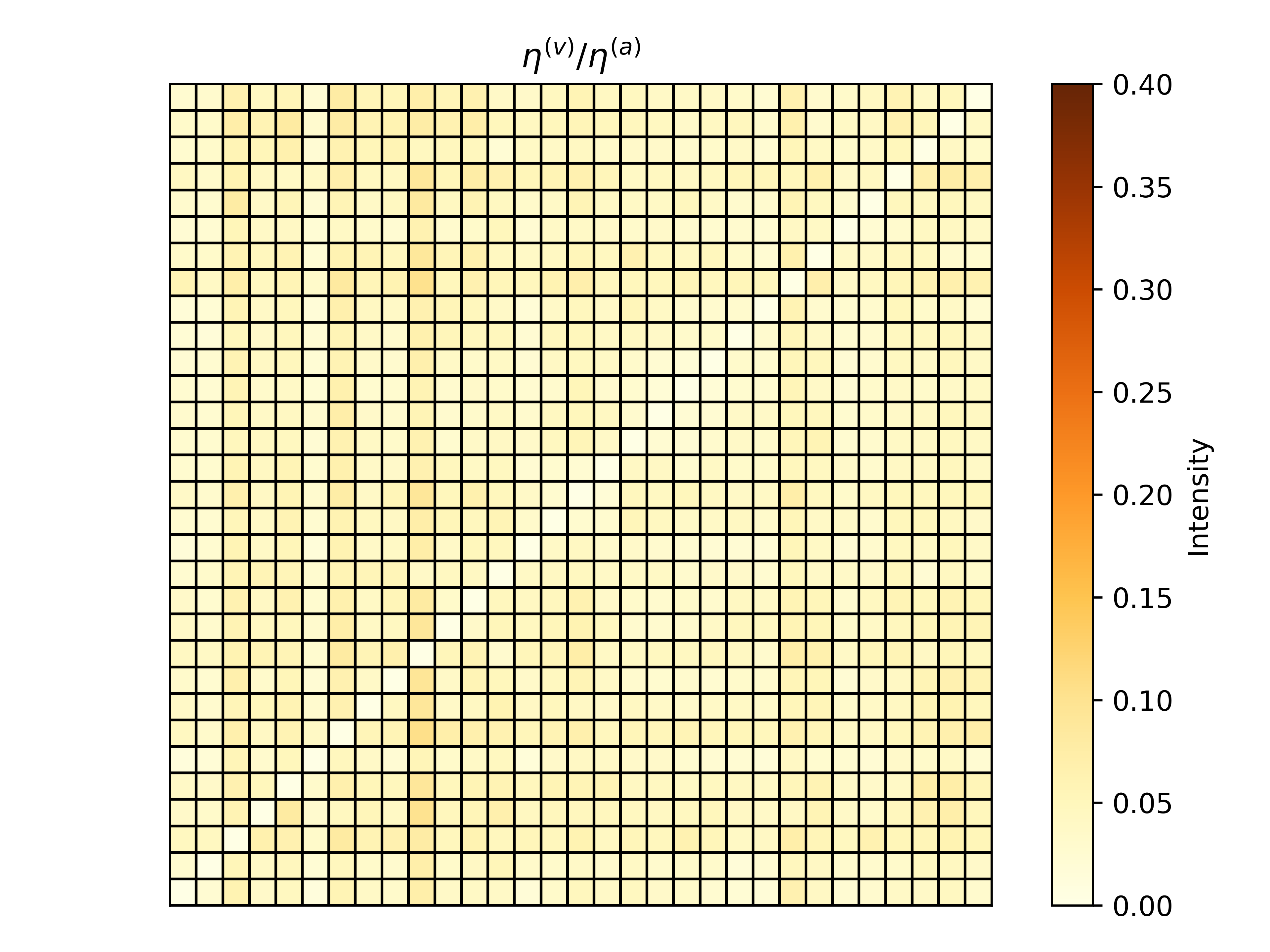}}
\subcaptionbox{$\eta^{(v)}/\eta^{(a)}$ in CRMT-JT.}{
\includegraphics[width=0.32\textwidth]{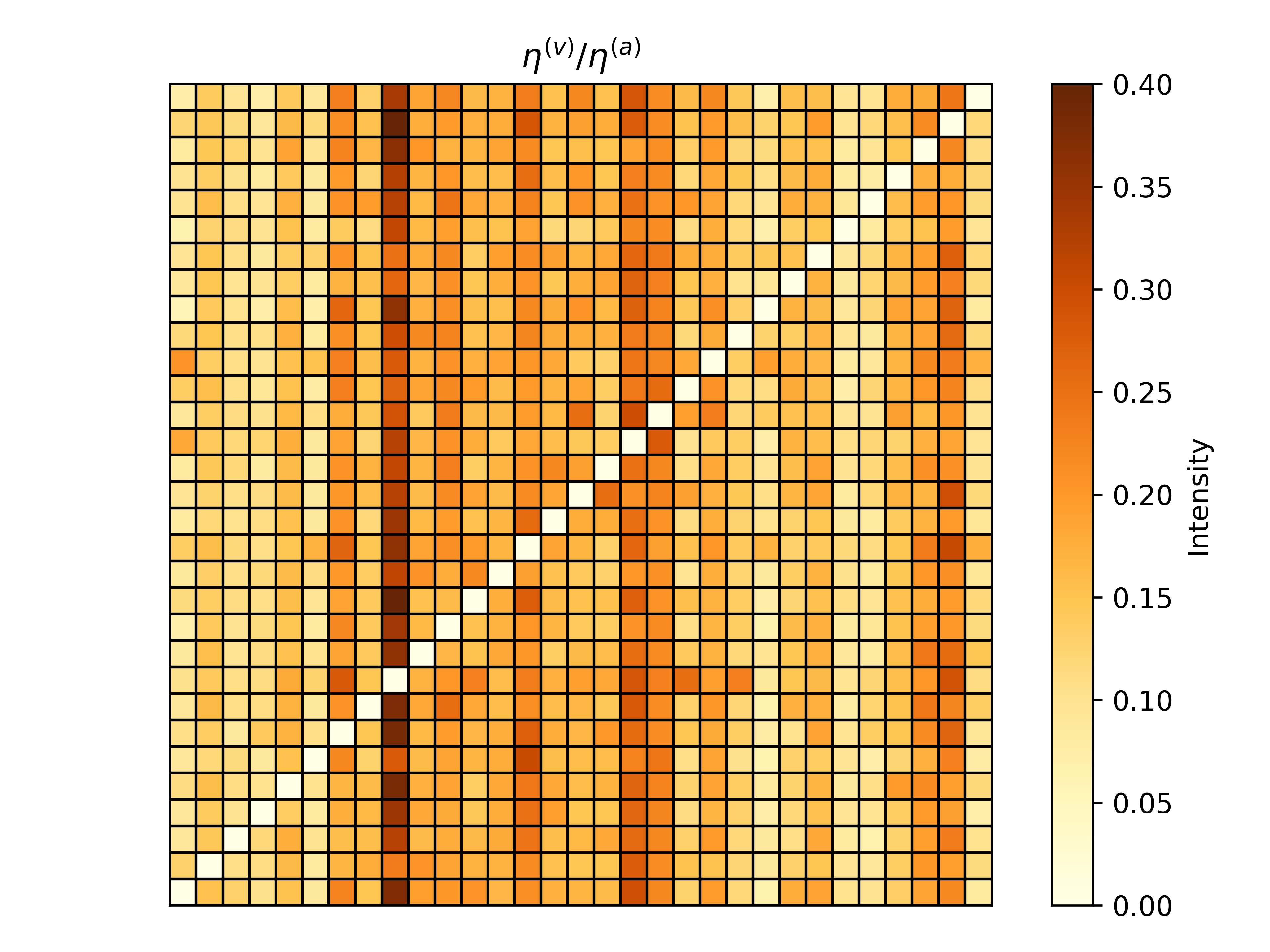}}
\caption{Experiment demonstrates the regularization of the vulnerability indicator in ablation studies.  }
\label{etas}
\end{figure}

\subsubsection{Complement for Ablation studies}

As shown in \autoref{ablation_studies}, we verify the effectiveness of each step of our training process. Here we provide more evidence that the improved multi-modal robustness is because we alleviate the influence caused by modality preference. Similar to the ablation studies, we report our findings related to solely using the orthogonal framework (OJT), our CRMT focusing exclusively on each step (CRMT-step-1 and CRMT-step-2), and our full method, CRMT-JT, which incorporates both steps. We also apply the ratio of uni-modal robustness indicators, to show the difference of resistance against uni-modal attack. From these heat maps of the ratio in \autoref{etas}, we have the following observations. First, since JT and OJT are different frameworks, their imbalance problem occurs in different class pairs (\textit{i.e.} entries in the heat map). And just applying the orthogonal framework could hardly solve the imbalance problem on the robustness indicator.  
Second, since the uni-modal representation may not be well learned, just applying step-2 of CRMT only brings marginal improvement. Third, CRMT step-1 can also alleviate the imbalance problem, since it enhances the learning of modality with less discriminative ability, and makes different modality representations more balanced. Fourth, equipped with both enlarging uni-modal representation margin and adjusting integration factors, our method can progressively alleviate the imbalanced problem in attack with different modalities, which further contributes to better robustness.

\end{document}